%% file: robust_influence.tex
\documentclass{article}

\usepackage{microtype}
\usepackage{graphicx}
\usepackage{subfigure}
\usepackage{booktabs} 
\usepackage{amsmath, amsfonts, amsthm, amssymb}

\usepackage{diagbox}
\usepackage{hyperref}

\newcommand{\F}{{\mathcal{F}}}
\newcommand{\e}{{\epsilon}}
\newcommand{\q}{{\theta}}
\DeclareMathOperator*{\E}{\mathbb{E}}
\newcommand{\p}{{\mathbf{p}}}
\newtheorem{definition}{Definition}
\newtheorem{theorem}{Theorem}
\newtheorem{lemma}{Lemma}
\newtheorem{corollary}{Corollary}

\newcommand{\Q}{{\mathcal{O}}}
\let\Pr\relax
\DeclareMathOperator*{\Pr}{\mathbb{P}}

 \newenvironment{itemize*}%
  {\vspace{-2ex} \begin{itemize} %
     \setlength{\itemsep}{-1ex} \setlength{\parsep}{0pt}}%
  {\end{itemize}}

\usepackage{color-edits}
\addauthor{td}{red}

\usepackage[accepted]{icml2019}

\icmltitlerunning{Robust Influence Maximization for Hyperparametric Models}

\begin{document}

\twocolumn[
\icmltitle{Robust Influence Maximization for Hyperparametric Models}




\begin{icmlauthorlist}
\icmlauthor{Dimitris Kalimeris}{goo}
\icmlauthor{Gal Kaplun}{goo}
\icmlauthor{Yaron Singer}{goo}
\end{icmlauthorlist}

\icmlaffiliation{goo}{Department of Computer Science, Harvard University, Cambridge, MA, USA}

\icmlcorrespondingauthor{Dimitris Kalimeris}{kalimeris@g.harvard.edu}
\icmlcorrespondingauthor{Gal Kaplun}{galkaplun@g.harvard.edu}

\icmlkeywords{Machine Learning, ICML}

\vskip 0.3in
]




\printAffiliationsAndNotice{}  

\begin{abstract}
In this paper we study the problem of robust influence maximization in the independent cascade model under a hyperparametric assumption. In social networks users influence and are influenced by individuals with similar characteristics and as such they are associated with some features. A recent surging research direction in influence maximization focuses on the case where the edge probabilities on the graph are not arbitrary but are generated as a function of the features of the users and a global hyperparameter. We propose a model where the objective is to maximize the worst-case number of influenced users for any possible value of that hyperparameter. We provide theoretical results showing that proper robust solution in our model is NP-hard and an algorithm that achieves improper robust optimization. We make-use of sampling based techniques and of the renowned multiplicative weight updates algorithm. Additionally we validate our method empirically and prove that it outperforms the state-of-the-art robust influence maximization techniques.
\end{abstract}

\input{intro}

\input{preliminaries}

\input{algorithm}
\input{mwu}

\input{lower_bound}

\input{experiments}

\input{conclusion}

\newpage
\setlength{\bibsep}{5pt}
\bibliography{bibliography}
\bibliographystyle{icml2019}

\newpage
\input{appendix}

\end{document}

%% file: intro.tex
\section{Introduction}\label{sec:intro}
In this paper we study robust influence maximization for hyperparametric diffusion models.  
First studied by Domingos and Richardson~\cite{domingos2001mining} and later elegantly formulated in seminal work by Kempe, Kleinberg, and Tardos~\cite{KKT03}, influence maximization is the algorithmic task of selecting a small set of individuals who can effectively spread information in a network.  The problem we formulate and address in this paper pertains to influence maximization for cases in which the information spread model in the network is subject to some uncertainty.

The most well studied model for information spread is the celebrated Independent Cascade (IC) model.  In this model the social network is modeled by a graph and every pair of nodes $u,v$ that are connected with an edge $e=(u,v) \in E$ are associated with a probability $p_{e}$ that quantifies the probability of $u$ spreading information to $v$.  Information spread in this model stochastically progresses from a set of nodes that initiates information to the rest of the nodes in the network as dictated by the graph topology and probabilities encoded on the edges.  Influence maximization is then the algorithmic task of selecting a fixed set of individuals that maximize the expected number of nodes that receive information.  More formally, given a graph $G=(V,E,\mathbf{p})$ where $V$ is the set of nodes, $E$ is the set of edges, and $\mathbf{p} \in [0,1]^{|E|}$ is the vector of edge probabilities, and a parameter $k\leq |V|$, influence maximization is the optimization problem:
$$\max_{S:|S|\leq k}f_{\mathbf{p}}(S)$$
where $f_{\mathbf{p}}(S)$ is the expected number of nodes in the network that receive information when $S$ is the initial set of nodes that spreads information.

In their seminal work, Kempe, Kleinberg, and Tardos proved that when $\mathbf{p}$ is known influence maximization can be reduced to monotone submodular maximization under a cardinality constraint.  Consequently, a simple greedy algorithm that iteratively selects the node whose marginal influence is approximately maximal obtains a solution that is arbitrarily close to a $1-1/e$ factor of optimal~\cite{NWF78}.





\subsection{Influence maximization under model uncertainty}
In recent years there has been a growing concern regarding the sensitivity of influence maximization to model uncertainty~\cite{GBL11, AKMKV13}.  In particular, small perturbations or uncertainty regarding the probability vector can have dramatic effects on the quality of a solution (see example illustrated in Figure~\ref{fig:uncertainty}), and even approximating the stability of an instance to small perturbations is  intractable~\cite{HK15}.

\subsection{Influence maximization under different models}

The sensitivity to small errors that we mentioned in the previous subsection is one issue. The second issue that one needs to consider is that we might actually have several models we want to optimize over. For example consider the case where a clothing company wants to advertise shirts and sweaters. The probabilities in the graph will be slightly different but we expect that the set of influential nodes will be more or less the same. Hence, it makes sense to try to identify a set of nodes that are influential, for all the underlying products, i.e. robust for the different models that we care (where each model induces different probabilities).

\begin{figure}
	\center
	\includegraphics[width=2.4in]{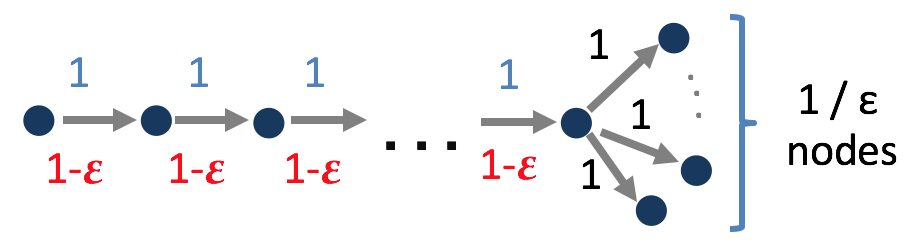}
    \caption{\footnotesize{This graph indicates the sensitivity of the influence maximization task to minor differences in the diffusion probabilities and hence the necessity of robust optimization. If all the diffusion probabilities are 1, selecting the node in the beginning of the chain maximizes the influence. However, if the edge probabilities are $1-\e$ it will be optimal to select the node in the end of the chain. 
    }}
        \label{fig:uncertainty}
\vspace{-0.3cm}
\end{figure}

\paragraph{Robust influence maximization.} To account for model uncertainty there has recently been a growing body of literature on \emph{robust} influence maximization~\cite{kempe,wei,nips17,OY17,nika} where the goal is to find a set of nodes whose influence is maximal over an entire set $\mathcal{P}$ of models:  
$$\arg\max_{S:|S|\leq k}\min_{\mathbf{p}\in \mathcal{P}} f_{\mathbf{p}}(S) $$
%

%
%
%

For a general set of models $\mathcal{P}$ it is easy to see that the robust influence maximization problem is either trivial or intractable.  Specifically, if we have confidence intervals for the diffusion probabilities, i.e. $p_e \in [c^{-}_e, c^{+}_e]$ for all edges $e$ then $\max_S\min_{\p}f_\p(S)$ simplifies into $\max_Sf_{\p^{-}}(S)$ where $p^{-} = (c^{-}_e)_{e\in E}$ due to the monotonicity of $f_{\mathbf{p}}$ in $\p$.  In general $\mathcal{P}$ can be exponentially large hence the problem is intractable.  Natural approaches like discretization, sampling, or maximin optimization over the influence function will fail to work for two computational reasons: 1) The space is of exponential dimension in $|E|$ and 2) The influence function is highly non concave-convex, a form of functions that is amenable to max-min optimization.  To circumvent these difficulties, previous work on robust influence maximization have taken two different approaches.  The first approach solves the max-min objective but assuming that the number of models is polynomial in the size of the problem (e.g.~\cite{nips17,nika}). The second focuses on the \emph{robust ratio} $\rho(S) :=\min\limits_{\mathbf{p}\in\mathcal{P}} f_\mathbf{p}(S)/f_{\mathbf{p}}(S_{\mathbf{p}}^\star)$ where $S_{\mathbf{p}}^\star$ denotes the optimal solution for $f_\mathbf{p}$~\cite{kempe, wei} which is a natural direction that comes with the caveat of not optimizing for the total number of nodes that are influenced.

However, optimizing for the robust ratio is not the right objective to consider and it can be up to a factor of $\sqrt{n}$ worse than the real robust solution as proved in the following lemma (proof in Appendix~\ref{app:proofs}) and illustrated in Figure~\ref{fig:ratio}.


\begin{lemma}\label{lemma:ratio_tight}
Let $\mathcal{P}$ be a set of influence functions. Consider the solutions to the  objectives: $\hat{S}_r = \arg\max_{S:|S| \leq k}\rho(S)$ and $\hat{S}_v = \arg\max_{S:|S| \leq k}\min_{\p \in \mathcal{P}}f_{\p}(S)$. There exists a set of influence functions $\mathcal{P}$ for which $\min_{\p}f_{\p}(\hat{S}_r) = \frac{1}{\sqrt{n}}\min_{\p} f_{\p}(\hat{S}_v)$, and this approximation ratio is tight.
\end{lemma}


\begin{figure}
	\center
   \includegraphics[width=3.in]{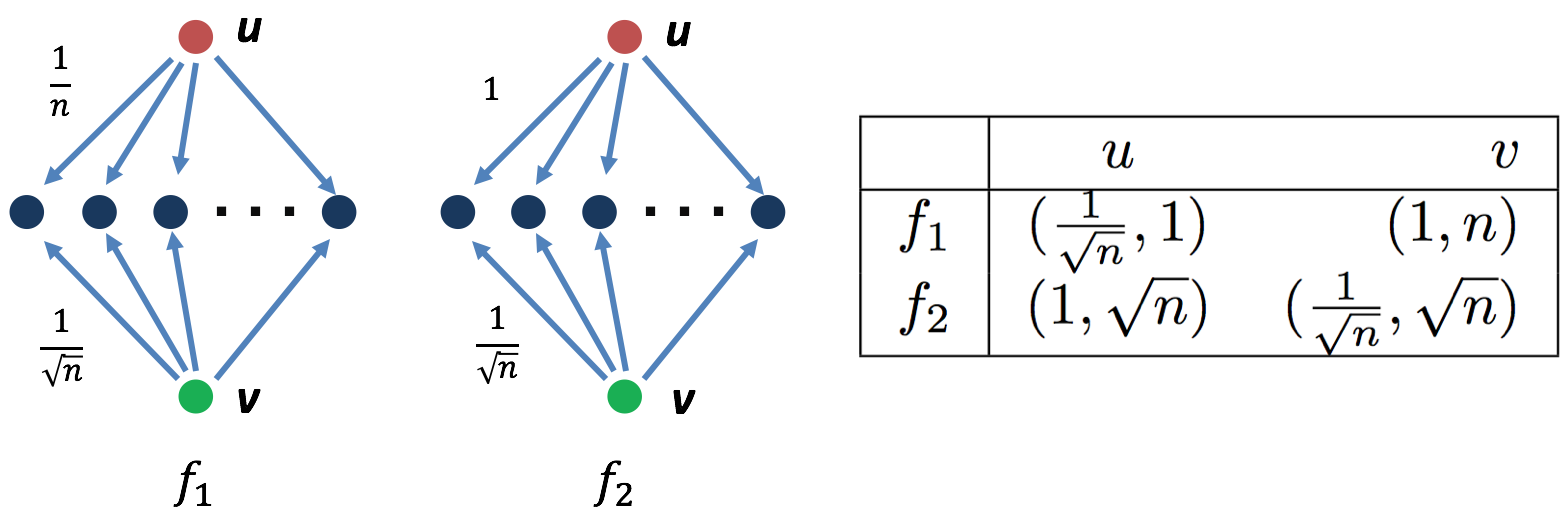} \\
   \caption{\footnotesize{Optimizing for the robust ratio may cost a \emph{multiplicative factor} of $\sqrt{n}$ in the worst case. Here, $n$ is the number of blue nodes and we consider  $k=1$. In $f_1$, the optimal solution is node $v$ while in $f_2$ it is $u$. The table shows the robust ratios and the expected number of influenced nodes for nodes $u$ and $v$ in the form $(ratio, \, expected)$. The ratio objective will choose node $u$, influencing only $1$ node in expectation. However, the direct approach will select node $v$ for $\sqrt{n}$ influenced nodes in expectation. \emph{Maximizing the robust ratio is not the right objective to consider.}}}
  \label{fig:ratio}
\vspace{-0.3cm}
\end{figure}

\paragraph{Robust optimization in hyperparameteric models.} A recent line of work in influence maximization and learning in networks explores the interaction of correlations of edge probabilities with the influence spread~\cite{bandits15,bandits17,reviewer1,icml18}. Specifically, it restricts the hypothesis class of IC by imposing correlations on the way that the probabilities are created. It assumes that each node in the network is associated with some features encoding information about it, for example social (age, etc.) or graph-related (degree, pagerank, etc.) characteristics. The influence probability between two nodes is a function of their features and a global \emph{low-dimensional} hyperparameter $\q$. Such approaches have been shown to have sample complexity that only depends on the dimension $d$ and importantly have been shown to be highly predictive on real information spread data collected from Facebook even with hyperparametric models parameterized by small number of dimensions~\cite{icml18}.  Intuitively, instead of searching individual probabilities we need to search for a hyperparameter in a much smaller space and once we find the right one it pins down all the influence probabilities.

Low-dimensional hyperparametric models circumvent the hardness associated with continuous spaces, as they impose structure and the complexity of the influence model is largely determined by the dimension $d$.  The main question we address in this paper can be informally stated as follows:
\begin{center}
\emph{Is there a computationally efficient algorithm to perform robust optimization for hyperparametric models?}
\end{center}

\subsection{Main result}
In this paper
we show that in contrast to general influence models, the hyperparamteric approach enables tractable solutions to the robust influence maximization problem.  At a high level, we show that by a simple sampling procedure one can find an \emph{efficient reduction} from continuous to discrete robust influence maximization that allows us to approximate the value $\max\min f$ instead of the ratio.


\subsection{Paper Organization} 
We begin by formalizing the robust influence maximization problem and the hyperparametric model in Section~\ref{sec:prelim}. In Section~\ref{sec:alg} we describe the main technical result of the paper which allows reducing the continuous robust optimization problem to a discrete problem.  In Section~\ref{sec:MWU} we use this reduction and introduce the Hyperparametric Influence Robust Optimizer (HIRO) algorithm for robust influence maximization. In Section~\ref{sec:lb} we provide a strong hardness result that shows the NP-hardness of robust optimization, even in the sense of bi-criteria approximation. Finally we evaluate the empirical performance of our algorithm in Section~\ref{sec:experiments}.

%% file: preliminaries.tex
\section{Preliminaries}
\label{sec:prelim}

A social network is modeled by a graph $G = (V, E)$ where $V$ is the set of individuals in the network and $E$ represents the friendships between them. We use $n$ and $m$ to denote $|V|$ and $|E|$ respectively. One of the core models for diffusion, the process through which information flows between the nodes of $G$, is the Independent Cascade model which was popularized in the seminal work of ~\cite{KKT03}.

\textbf{The Independent Cascade (IC) model.} The IC model describes a discrete-step stochastic process through which diffusion spreads from a set of initially active individuals to the rest of the nodes in the network. Each node can be active or inactive and each edge $e\in E$ in the network is associated with some probability $p_e$.  All nodes begin as inactive and at time step $t=0$ a subset of nodes $S \subseteq V$, called the seed set, is chosen and becomes active. At every time step $t + 1$, every node $u$ that became active at time step $t$ attempts to influence every of its non-active neighbors $v$, independently and succeeds with probability $p_{(u,v)}$. 

\textbf{Influence functions.} 
For a given graph $G=(V,E)$ and vector of probabilities $\mathbf{p} \in [0,1]^{m}$ the \emph{influence function} $f_{\mathbf{p}}:2^V\to \mathbb{R}$ measures the expected number of nodes that will become influenced in the graph $G$ for a seed set $S \subseteq V$:
$$f_{\mathbf{p}}(S) = \sum_{A\subseteq E}r_{A}(S)\prod_{e \in A}p_e\prod_{e \notin A}(1-p_e)$$
where $r_A(S)$ denotes the number of nodes that are reachable in $G$ from $S$ using only edges from $A$. An important property of $f_{\mathbf{p}}$ is that it is \emph{monotone submodular} for any $\mathbf{p}$.  

\paragraph{Influence maximization.} For a given influence function $f_{\mathbf{p}}:2^{[n]} \to \mathbb{R}$ and value $k\leq n$, influence maximization is the optimization problem: $\max_{S:|S|\leq k}f_{\mathbf{p}}(S)$.  The problem is NP-hard but since the influence function is monotone and submodular~\cite{KKT03} a simple greedy algorithm which iteratively selects nodes whose marginal contribution is largest obtains a solution that is a $1-1/e$ factor of the optimal solution~\cite{NWF78} and this approximation ratio is optimal unless P=NP~\cite{Feige98}.



\textbf{Robust influence maximization.} 
Given a graph $G=(V,E)$ and a set of probability vectors $\mathcal{P}$ the goal of \emph{robust} influence maximization is to find a solution of size $k$ that has high value for every possible influence function $f_\mathbf{p}$ that can be generated by $\mathbf{p} \in \mathcal{P}$:
\vspace{-.1cm}
$$\max_{S:|S|\leq k}\min_{\mathbf{p} \in \mathcal{P}}f_{\mathbf{p}}(S)$$
%
There are two sources of inapproximability known for this problem: first, since it a generalization of influence maximization it is NP-hard to get any approximation better than $1-1/e$ and we are therefore satisfied with solutions that are \emph{approximately} optimal.  The second source of inapproximability is due to the fact that our solution space $\{S:|S|\leq k\}$ is highly non-convex which makes it intractable to obtain \emph{proper} solutions (Figure~\ref{fig:improper}). In particular, it is NP-hard to find a set of size $k$ that obtains any approximation better than $\mathcal{O}(\log (n))$ for the robust optimization problem~\cite{nips17}.  For this reason, we seek \emph{bi-criteria approximations}.  A solution $\hat{S}$ is an $(\alpha,\beta)$ bi-criteria approximation to the max-min solution of size $k$ if $\beta |\hat{S}| \leq k$ and:
$$ \min_{\mathbf{p} \in \mathcal{P}}f_{\mathbf{p}}(\hat{S}) \geq \alpha \max_{S:|S|\leq k}\min_{\mathbf{p} \in \mathcal{P}}f_{\mathbf{p}}(S)$$
Due to its sources of hardness the gold standard for robust influence maximization are $\left (1-1/e,\Omega  (\log^{-1}(n) ) \right)$ bicreteria approximations~\cite{nips17, krause, kempe, nika}.  


\textbf{Hyperparametric influence models.} 
A hyperparametric model $H:\Theta\times X \to [0,1]$ restricts the traditional IC model by imposing correlations between the probabilities of different edges.  Each edge $e = (u,v)$ is associated with a $d$-dimensional feature vector $x_e \in X \subseteq [-1,1]^d$ encoding information about its endpoints. The probability of $u$ influencing $v$ is a function of the features $x_{(u,v)}$, and a global hyperparameter $\theta \in \Theta \subseteq [-B,B]^d$, for some constant $B >0$. That is: $p_e = H(\theta, x_e)$.  
The most standard hyperparametric models are Generalized Linear Models (GLM) for which:
\vspace{-.1cm}
$$H(\theta,x_{e}) = h(\theta^\top x_{e}) + \xi_e$$
 %
 where $\xi_e$ is drawn from some bounded distribution.  To ease the presentation, throughout this paper we treat the model as if $\xi_e = 0$\footnote{We note that the results carry over when $\xi_e$ is drawn  from a distribution with mean 0 and bounded support. 
 }.  Our results hold for a family of generalized linear models that we define later which includes standard choices for $h$ are linear, sigmoid, or the logarithm functions.

\begin{figure}
	\center
	\includegraphics[width=1.5in]{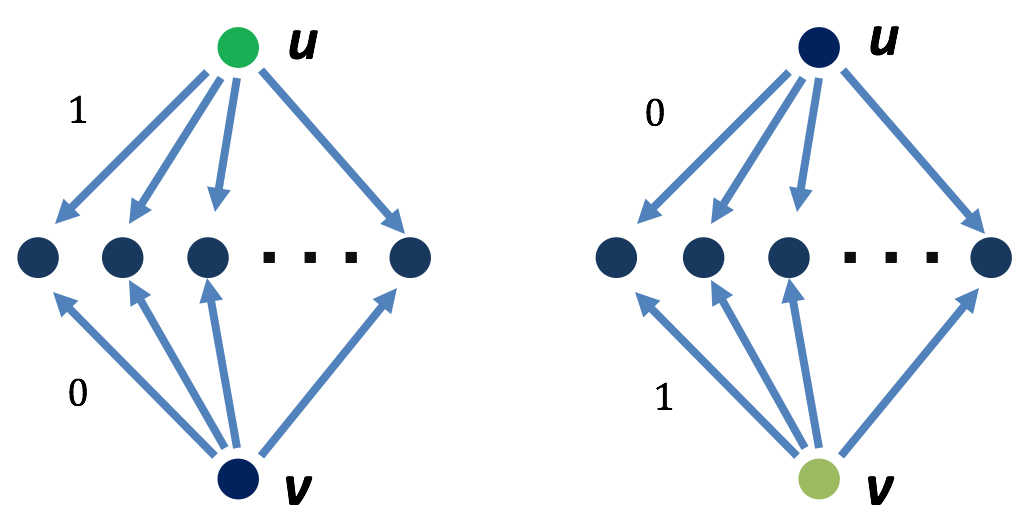}
    \caption{The green node denotes the optimal solution in each case. Any deterministic solution has robust influence of 1 (only the selected node). However, the distribution selecting uniformly from $\{u,v\}$ (improper solution) has expected influence of $n/2$.}
        \label{fig:improper}
\vspace{-0.3cm}
\end{figure}

\textbf{Hyperparameter Influence Robust Optimization.} 
For a given hyperparametric model $H$, set of features $\{x_{e}\}_{e\in E}$ and hyperparameter space $\Theta \subseteq [-B,B]^d$ the set of possible diffusion probabilities is:
$$\mathcal{H} := \{      \{ H( \theta,x_e) \}_{e\in E} \ : \ \theta \in \Theta    \}$$
and robust influence maximization then reduces to:
%
\begin{align*}
\min_{\mathbf{p} \in  \mathcal{H}} f_{\mathbf{p}}(\hat{S}) \geq \alpha \max_{S:|S|\leq k}\min_{\mathbf{p} \in \mathcal{H}}f_{\mathbf{p}}(S)
\end{align*}
\vspace{-.2cm}

Throughout the paper we prove theorems for general feature spaces and hyperparametric models.  When the hyperparametric model $H$ and set of features is clear from context, it will be convenient to use $f_{\theta}$ instead of $f_{\mathbf{p}}$ where $\mathbf{p}$ is the the probability vector $\mathbf{p}$ generated by a hyperparametric model.  We will use the abbreviated notation defined above:
$$ \min_{ \theta \in \Theta  }f_{\theta}(\hat{S}) \geq \alpha \cdot \max_{S:|S|\leq k}\min_{\theta \in \Theta}f_{\theta}(S)$$

%


%% file: algorithm.tex
\section{Robust Optimization}
\label{sec:alg}
In this section we describe the main result of the paper.  We show that for an extremely broad class of hyperparametric models robust influence maximization is computationally tractable.  More specifically, we show that for generalized linear hyperparameteric models that are 1-Lipschitz, for any $\epsilon>0$, 
a natural sampling procedure from the hyperparameter space $\Theta = [-B,B]^d$ generates $l \in \tilde{\mathcal{O}} \left ( d\left( \frac{n\cdot m\cdot d}{\epsilon} \right)^d \log \frac{1}{\delta} \right )$ influence functions $f_1,\ldots,f_{l}$ such that with probability $1-\delta$, robust continuous optimization over $\Theta$ reduces to robust discrete optimization over the functions $\{f_i\}_{i=1}^l$. That is, an algorithm that returns $\hat{S} \subseteq V: \,|\hat{S}|\leq k$ satisfying: 
\vspace{-.15cm}
$$
\min_{i\in[l]}f_i(\hat{S}) \geq \alpha\max_{S:|S|\leq k}\min_{i \in [l] }f_{i}(S)
$$
\vspace{-.15cm}
 implies the existence of an algorithm for which:
\vspace{-.15cm}

$$
\min_{\q\in\Theta}f_\q(\hat{S}) \geq \alpha\max_{S:|S|\leq k}\min_{\theta\in \Theta }f_{\theta}(S) -\e.
$$
%
This reduction from the continuous space of infinitely-many functions $\{f_{\theta}\}_{\theta \in \Theta}$ to polynomially-many functions $\{f_{i}\}_{i\in [l]}$ is handled in two steps.  We first prove that influence functions that are generated by a class of hyperparametric models that we call \emph{stable} have bounded Lipschitzness (Section \ref{sec:lipschitz}).  Using this property we then prove that sampling polynomially-many functions from the hyperparametric model suffices to obtain approximation to the robust objective (Section~\ref{sec:sampling}).  Finally, we show how to produce near optimal solutions to the robust optimization problem defined on $\{f_{\theta}\}_{\theta \in \Theta}$ by implementing a best-response oracle on a set of sampled functions $\{f_{i}\}_{i\in [l]}$ using a Multiplicative Weight Updates (MWU) procedure (Section \ref{sec:MWU}).

\subsection{Stability implies Lipschitzness}\label{sec:lipschitz}
We now prove that $f_{\theta}$ is $L$-Lipschitz for $L \in \texttt{poly}(n)$ if the hyperparametric model that generates it is \emph{stable}.


\begin{definition}
A hyperparametric model $H: \Theta \times X \to [0,1]$ is \emph{stable} if it is a generalized linear model that is 1-Lipschitz with respect to the $\ell_1$ norm, i.e. for $\theta\in \Theta, x_e \in X$ we have that $H(\theta,x_e) = h(\theta^\top x_e )$ and for every $\theta,\theta' \in \Theta$: 
$$| h  (\theta^\top x_e ) - h(\theta'^\top x_e )| \leq \| \theta - \theta' \|_1$$
\end{definition}
It is easy to verify that the hyperparametric models used in influence maximization literature are stable: linear $H(\theta,x_e) = \theta^\top x_e$, logistic $H(\theta,x_e) =  \frac{1}{1 + \exp(- \theta^\top x_e)}$ and probit $H(\theta,x_e) = \Phi(\theta^\top x_e)$ where $\Phi$ is the CDF of the standard Gaussian, appear in~\cite{bandits15, bandits17, icml18} and are all stable.


Intuitively, stable hyperparametric models are not sensitive to small changes of the hyperparameter. Hence, despite the fact that a modification on $\theta$ affects the probabilities in all edges, the difference is not very large and we are able to bound the absolute change in the influence function. The following two lemmas are inspired by~\cite{wei} and \cite{wei2} who proved similar results for non-hyperparametric models.


\begin{lemma}\label{lem:lipschitz}
Assume that the hyperparametric model is stable. Then, the influence function $f_\theta$ is Lipschitz with respect to the $\ell_1$ norm with Lipschitz constant $L = nm$.
\end{lemma}

\begin{proof}
Since $H$ is stable we get that for every $e \in E$ and every two edge probabilities $p_e, p'_e$ produced by the hyperparametric model by parameters $\theta,\theta' \in \Theta$: 
\begin{align}
|p_e - p'_e| & = |H(\theta, x_e) - H(\theta', x_e)| \\ 
& = |h(\theta^\top x_e) - h(\theta'^\top x_e)| \leq \|\theta - \theta'\|_1. 
\end{align}
Hence, $\|\theta - \theta'\|_1 \leq \e$ implies $|p_e - p_e'| \leq \e$.

Now notice that the influence function is monotone with respect to the diffusion probabilities. 
 As a result, the maximum change that can occur given the constraint $|p_e - p_e'| \leq \e$ is when $p_e' = \min\{p_e + \e,1\}$ (or $p_e' = \max\{0, p_e - \e\}$) for all $e \in E$. We focus on the first case; the second is identical.

Fix a seed set $S$. When the probability of an edge $e = (u,v)$ is increased from $p_e$ to $p_e+\e$, there is an increase in activation probability for all the nodes that are reachable from $S$ through $u$ by at most $\e$. There are at most $n$ nodes that are reachable through $u$ so the total change in the influence of set $S$ in that case is $n\e$. 
Using the same argument for each edge in the network gives the desired bound of $nm\e$.
\end{proof}

In our method, the Lipschitz parameter is polynomially related to the complexity of implementing a best-response oracle in the MWU procedure.  The fact that the Lipschitzness $L$ is polynomial in the number of nodes in the graph $n$ is, therefore, a crucial property as it makes the robust optimization problem computationally tractable.

\textbf{Tightness of the Lipschitz constant.} The Lipschitzness of the function polynomially determines the computational complexity of our method.  As we show, the Lipschitzness bound $L=n\cdot m$ is tight in the worst case over all possible graphs (proof is deferred to Appendix A). 



\begin{lemma}\label{lemma:tight_example}
There is a graph for which any influence function generated by a stable non-trivial\footnote{$h:\mathbb{R}\rightarrow[\lambda, 1-\lambda]$ is non-trivial if it is continuous and surjective for some $x_e\in X$.} generalized linear model that is continuous has Lipschitzness $L = n m$.
\end{lemma}

\subsection{Covering $\Theta$ via sampling}\label{sec:sampling}
The Lipschitzness property from Lemma~\ref{lem:lipschitz} is useful since it bounds the change in the estimation of influence function $f_{\theta}$ when we approximate it with another function $f_{\theta'}$ for $\theta,\theta'$ that are $\epsilon$-close, i.e. $\|\theta - \theta' \|_1 \leq \epsilon$.  Thus, if we can construct a set of representative parameters $\Theta_{\epsilon}$ s.t. $\forall \theta \in \Theta$ there is some $\theta'$ which is $\epsilon$ close with respect to the $\ell_1$ norm we can reduce the robust optimization problem over the infinite hyperparameter space to the finite case.  We know now this can be done using a reasonable number of samples from $\Theta$.

\begin{lemma}\label{lemma:sampling}
Let $\Theta = [-B, B]^d$ and $\e, \delta > 0$. If we sample a set $\Theta_\e$ of size $s \in \tilde{\Q} \left (d \left (\frac{Bd}{\epsilon}\right)^d   \log\frac{1}{\delta} \right )$ uniformly at random from $\Theta$, then with probability at least $1-\delta$, for any $\q \in \Theta$ there exists $\q' \in \Theta_\e$ such that $\|\q - \q'\|_1 \leq \e$.
\end{lemma}

\begin{proof}
It suffices to partition the hypercube $[-B, B]^d$ into $\ell_1$ balls of radius $\e$ and bound the number of points necessary to sample in order to have at least one point in each $\ell_1$ ball. Since every point in $\Theta$ is covered by a ball and every two points in the ball are close, sampling points that correspond to every ball is a representative set of parameters that are $\epsilon$-close w.r.t. the $\ell_1$ distance to any parameter in $\Theta$. 

It is well known that $r = \left(\frac{2B\cdot d}{\epsilon} \right )^d$ balls of $\ell_1$ radius $\epsilon$ suffice to cover the $d$-dimensional hypercube $[-B, B]^d$.  Let $b_1,\ldots,b_r$ denote the covering balls and $\theta_1,\ldots,\theta_{s}$ be points drawn uniformly at random from $\Theta = [-B, B]^d$.  Then:
\begin{align*}
\Pr[\exists b_j \text{ such that } \nexists  \theta_i \in b_j]
& \leq \sum_{j}\Pr[\nexists  \theta_i \in b_j]\\
& \leq r\Pr[\nexists  \theta_i \in b_1] \\
& \leq r\left(1-\frac{1}{r}\right)^s
\end{align*}
%
Using $s \in \Q\left(d(\frac{Bd}{\epsilon})^d\log\frac{Bd}{\e\delta}\right)$ samples we get:
$$\Pr[\exists b_j \text{ such that } \nexists  \theta_i \in b_j]  \leq \delta$$
Thus, since $\Theta \subseteq \cup_j b_j$ and the radius of the balls is $\epsilon$ we get that for every $\theta \in \Theta$ there exists a $\theta' \in \{\theta_i\}_{i \in [s]} $ that is $\epsilon$ close w.r.t. $\ell_1$ norm with probability at least $1-\delta$.
\end{proof}

Furthermore, we can extend the notion of \emph{covering} a convex space $\Theta$ to the coverage of a family of functions.

\begin{definition}\label{def:cover}
Let $\mathcal{F} = \{f_\q:2^V\to \mathbb{R}\,|\, \q \in \Theta\}$ be a family of influence functions and $\mathcal{F}_\e \subset \mathcal{F}$ such that $|\mathcal{F}_\e| < \infty$. We say that $\mathcal{F}_\e$ $\e$-\emph{covers} $\mathcal{F}$ if for any $f \in \mathcal{F}$  there exists an $f_\e \in \mathcal{F}_\e$ and any $S \subseteq V$ s.t. $|f(S) - f_\e(S)| \leq \e$.
\end{definition}

The following corollary is obtained by fusing lemmas~\ref{lem:lipschitz} and \ref{lemma:sampling} with Definition~\ref{def:cover}.

\smallskip
\begin{corollary}\label{cor:cover}
Let $\mathcal{F} = \{f_\q:2^V\to \mathbb{R}\,|\, \q \in \Theta\}$ be the family of influence functions and let $\mathcal{F}_\e$ be sampled uniformly at random from $\mathcal{F}$, such that $|\mathcal{F}_\e| \in  \tilde{\mathcal{O}} \left (d\left(\frac{LBd}{\epsilon}\right)^d  \log\frac{1}{\delta} \right )$. Then $\mathcal{F}_\e$ $\e$-covers $\mathcal{F}$ with probability at least $1-\delta$.
\end{corollary}

%% file: mwu.tex

\section{Reducing Continuous to Discrete RO}\label{sec:MWU}


\begin{algorithm}[tb]
   \caption{HIRO: Hyperparam Inf Robust Optimizer}
   \label{alg:inf_max}
\begin{algorithmic}
   \STATE {\bfseries Input:} 
   ${G=(V,E),
   \{x_e\}_{e \in E},
   H:\Theta\times X \to [0,1], 
   \e, \delta}$
   \STATE $\,$
   \STATE $l \leftarrow \tilde{\mathcal{O}} \left (d\left(\frac{LBd}{\epsilon}\right)^d  \log\frac{1}{\delta} \right )$
   \STATE $T \leftarrow \tilde{\Q}\left(\frac{d(\log n + \log\log\frac{1}{\delta})}{\e^2}\right)$, $\,\,\eta \leftarrow \frac{\log l}{2T}$
   \STATE $\,$
   \STATE $f_{1},\ \ldots, \ f_{l}  \hspace{0.27in}      \leftarrow \textsc{Sample}(G,H,\Theta,\{x_e\}_{e \in E})$
   \STATE $w_1[1],\ldots,w_{1}[l] \leftarrow 1/l,\ldots,1/l$ 
   \FOR{each time step $t\in [T]$}
   \STATE $S_t \leftarrow \textsc{Greedy}(\sum_{i=1}^l w_t[i]f_{i})$
   \FOR{each $i\in [l]$} 
   \STATE $$w_t[i] \propto \text{exp}\left\{-\eta\sum_{\tau=1}^{t-1}f_{i}(S_\tau)\right\}$$
   \ENDFOR
   \ENDFOR
   \STATE {\bfseries Output:} select $S$ u.a.r. from $\{S_1, S_2, \ldots, S_T\}$
\end{algorithmic}
\end{algorithm}

In the previous section we proved that since the influence function is Lipschitz, the infinite family of functions $\mathcal{F} = \{f_\q\,|\, \q \in \Theta\}$ can be well-approximated by the finite (and crucially, polynomially-sized) family $\mathcal{F}_\e = \{f_{\q_\e}\,|\, \q_\e \in \Theta_\e\}$, where $\Theta_\e$ is obtained by sampling u.a.r. from $\Theta$. Hence, the task of robust influence maximization in the hyperparametric model intuitively reduces to finding a procedure to perform robust optimization on a finite set of functions.



The following lemma formalizes this intuition by stating that, under mild conditions, $\alpha$-approximate continuous robust optimization reduces to $\alpha$-approximate discrete robust optimization. The proof is deferred to Appendix A.

\begin{lemma}\label{lemma:reduction}
Let $\mathcal{F} = \{f_\theta:2^V\to \mathbb{R}\,|\, \theta \in \Theta \subseteq [-B,B]^d\}$ be a family of influence functions. Consider a family $\mathcal{F}_\e \subset \mathcal{F}$ s.t. $\F_\e$ $\e$-covers $\F$. Then, $\alpha$-approximate robust optimization on $\mathcal{F}$ reduces to $\alpha$-approximate robust optimization on $\F_\e$. That is, an algorithm that returns $\hat{S} \subseteq V$:
\vspace{-.1cm}
$$\min_{f_{\theta}\in \F_\e}f_{\theta}(\hat{S}) \geq \alpha\cdot \max_{S\subseteq V}\min_{f_{\theta} \in \F_{\epsilon}}f_{\theta}(S)$$
\vspace{-.1cm}
implies an algorithm that for any $\e > 0$ returns $\hat{S} \subseteq V$ s.t.:
\vspace{-.1cm}
$$\min_{f_{\theta} \in \mathcal{F}}f_{\theta}(\hat{S}) \geq \alpha\cdot \max_{S\subseteq V}\min_{f_{\theta} \in \mathcal{F}}f_{\theta}(S) - \e.$$
\vspace{-.1cm}
\end{lemma}



%



\paragraph{The HIRO algorithm.}  Given the reduction from the continuous hyperparametric problem to the discrete we can now describe the Hyperparameteric Influence Robust Optimizer ($\textsc{HIRO}$) Algorithm~\ref{alg:inf_max} which gives an optimal bi-criteria approximation to the robust influence maximization in the hyperparametric setting.  $\textsc{HIRO}$ takes as input a graph $G$, the edge features $\{x_e\}_{e \in E}$, the hyperparametric model $(H, \Theta)$ that dictates the edge probabilities and an error parameter $\e$ that controls the quality of the returned solution. The output is a set of $k$ nodes. It starts by sampling $l$ points from $\Theta$ and hence, constructing $l$ different influence functions that serve as a proxy for the continuous problem. It assigns uniform weights to these functions and runs MWU for a number of steps that depends on $\e$, where in each step higher emphasis is placed on the ones with poor historical performance. 
In every iteration it optimizes a convex combination of the functions, which is possible since all $f_i$s are monotone submodular, and hence amenable to optimization using the \textsc{GREEDY} algorithm. It keeps the outcome as a candidate solution. In the end, one of the candidate solutions is returned u.a.r. The intuition is that since each solution is good for some iteration of the algorithm (meaning for some specific weighting of the $f_i$s), on expectation the solution that is returned is good for all the functions that performed poorly for some iteration of the MWU and hence, robust.


\begin{theorem}\label{thm:robust}
$\textsc{HIRO}$ with error parameter $\e$, runs in time $\texttt{poly}(n, \e, \log(1/\delta))$ and returns the uniform distribution $\mathcal{U}$ over solutions $\{S_1, \ldots, S_T\}$, s.t. with probability at least $1 - \delta$:
$$
\min_{\theta\in \Theta}\E_{S\sim\mathcal{U}}[f_{\theta}(S)] \geq \left(1 - \frac{1}{e}\right)\max_{S:|S|\leq k}\min_{\theta\in \Theta}f_\theta(S) - 2\e.
$$
where $T \in \tilde{\Q}\left(\frac{d(\log n + \log\log\frac{1}{\delta})}{\e^2}\right)$.
\end{theorem}


\begin{proof}
From Corollary~\ref{cor:cover} we know that constructing $l$ different influence functions $f_1, \ldots, f_l$ by sampling $l = \tilde{\mathcal{O}} \left (d\left(\frac{LBd}{\epsilon}\right)^d  \log\frac{1}{\delta} \right )$ points u.a.r. from $\Theta$ yields an $\e$-cover of $\Theta$ with probability at least $1-\delta$.

Since $L \leq n\cdot m$ (Lemma~\ref{lem:lipschitz}) and $d$ is constant according
to the main assumption of the hyperparametric model, we know that constructing the cover takes polynomial time.

\cite{nips17} proved that for influence functions $f_1, \ldots, f_l$, the MWU procedure, run for $T$ iterations, with a learning rate of of $\eta = {\log(l)}/{2T}$ that uses $\textsc{GREEDY}$ as an approximate best-response oracle, returns the uniform distribution over $T$ solutions s.t.:


\vspace{-.3cm}
\small{\begin{multline*}
\min_{i \in [l]}\E_{\hat{S}}[f_i(\hat{S})] \geq \left(1-\frac{1}{e}\right)\max_{S}\min_{i \in [l]}f_i(S) - \Q\left(\sqrt{\frac{\log(l)}{T}}\right)
\end{multline*}}
By applying Lemma~\ref{lemma:reduction} we get:
\vspace{-.1cm}
\small{\begin{multline*}
\min_{\theta \in \Theta}\E_{\hat{S}\sim U}[f_\q(\hat{S})] \geq \left(1-\frac{1}{e}\right)\max_{S}\min_{\q \in \Theta}f_\q(S) \textrm{-} \e \textrm{-}\Q\left(\sqrt{\frac{\log(l)}{T}}\right)
\end{multline*}}
\vspace{-.1cm}

For \small{$l \in \tilde{\Q}\left(d\left(\frac{LBd}{\epsilon}\right)^d\log\frac{1}{\delta}\right)$} we get the desired bound by setting $T \in \tilde{\Q}\left(\frac{d(\log n + \log\log\frac{1}{\delta})}{\e^2}\right)$ as required, since $L\cdot B \in \Q(n)$.
%
%
%
%
%
%
%
\end{proof}

The result of the previous theorem implies the following bi-criteria approximation guarantee by returning $\hat{S} = \cup_{t=1}^TS_t$ instead of the uniform distribution over the $\{S_1, \ldots, S_T\}$.

\bigskip

\begin{corollary}\label{cor:bicriteria}
\textsc{HIRO} with error parameter $\e$, returns a set $\hat{S}$ of size at most $\tilde{\Q}\left(\frac{d(\log n + \log\log(1/\delta))}{\e^2}k\right)$ which, with probability at least $1-\delta$, satisfies:
$$\min_{\theta\in \Theta}f_{\theta}(\hat{S}) \geq \left(1 - \frac{1}{e}\right)\max_{S:|S|\leq k}\min_{\theta\in \Theta}f_\theta(S) - 2\e.$$
\end{corollary}

%% file: lower_bound.tex
\section{Lower Bound}
\label{sec:lb}

In this section, we prove that a structural assumption such as the hyperparametric restriction of the IC model is vital for robust influence maximization, since otherwise we might need to sample a number of functions that is exponential in $n$. We do so by providing a strong hardness result: define the family of functions $\F_\p = \{f_\p\,|\,\p \in \mathcal{P}\}$ for some finite $\mathcal{P}$.
  We prove that the problem $\max_{S: |S| \leq k}\min_{\mathbf{p} \in \mathcal{P}}f_{\mathbf{p}}(S)$ is NP-hard to approximate within any constant factor with a reasonable bicriteria approximation.

This is in sharp contrast with the main result about hyperparametric robust influence maximization that we proved in Theorem~\ref{thm:robust} and Corollary~\ref{cor:bicriteria} since, assuming the hyperparametric model, $\mathcal{P}$ is always of polynomial size and hence we only need to increase our budget by a factor of $\tilde{\mathcal{O}}(\log n)$. In Theorem~\ref{thm:imposibility} we formalize the impossibility result. We provide a proof sketch, the full proof is in Appendix B.

\begin{theorem}\label{thm:imposibility}
Let $G$ be a graph with $n$ nodes and $m$ edges, and $\delta, \e > 0$. There is no algorithm to find a set $\hat{S}$ of size $|\hat{S}| \leq (1-\delta)\ln |\mathcal{P}| \cdot k$ that achieves an approximation factor better than $\Q\left(\frac{1}{n^{1-\e}}\right)$ to the problem $\max_{S:|S|\leq k}\min_{\mathbf{p} \in \mathcal{P}}f_\mathbf{p}(S)$, where $\mathcal{P} \subseteq \{\lambda, 1-\lambda\}^m$, $\lambda = o(1)$, and $|\mathcal{P}| \in \Omega(\texttt{poly}(n))$, unless $P = NP$.
\end{theorem}

This essentially means that if in our problem, we need to construct a cover that contains more than a polynomial number of functions then, the best bicriteria approximation we can hope for will contain significantly more than $k$ nodes.

%
%
%
%

\begin{proof}


Our reduction is based on \cite{kempe}. The authors there prove the hardness of a different version of robust influence maximization where the objective is to approximate well the individual optima for a set of different influence functions instead of influencing as many nodes as possible in the worst case. 
 We reduce from GAP SET COVER. The use of the gap version of the problem is to show that even if we augment the budget of nodes by a factor of roughly $\log |\mathcal{P}|$ the problem remains NP-hard.

Given an instance of gap set cover, we construct a bipartite graph on $n$ nodes, $m$ edges, and $\texttt{poly}(n)$  different probability sets that correspond to the different influence functions, such that when there is a set cover of size at most $k$ then there exists a seed set $S$ for which all the influence functions have high value, while when there is no set cover of size $\log n \cdot k$, then there is at least one influence function that has low value for any seed set, even if we allow sets of size $\log n \cdot k$. The asymptotic difference between the values of the objective functions, with and without the cover set, enables the decision of the gap set cover.\end{proof}

%% file: experiments.tex
\section{Experiments}
\label{sec:experiments}

\begin{figure*}[h!t!]
\begin{center}
\begin{tabular} {cccc}
  \includegraphics[width=1.4in, trim={0cm 0 1.64cm 0}, clip]{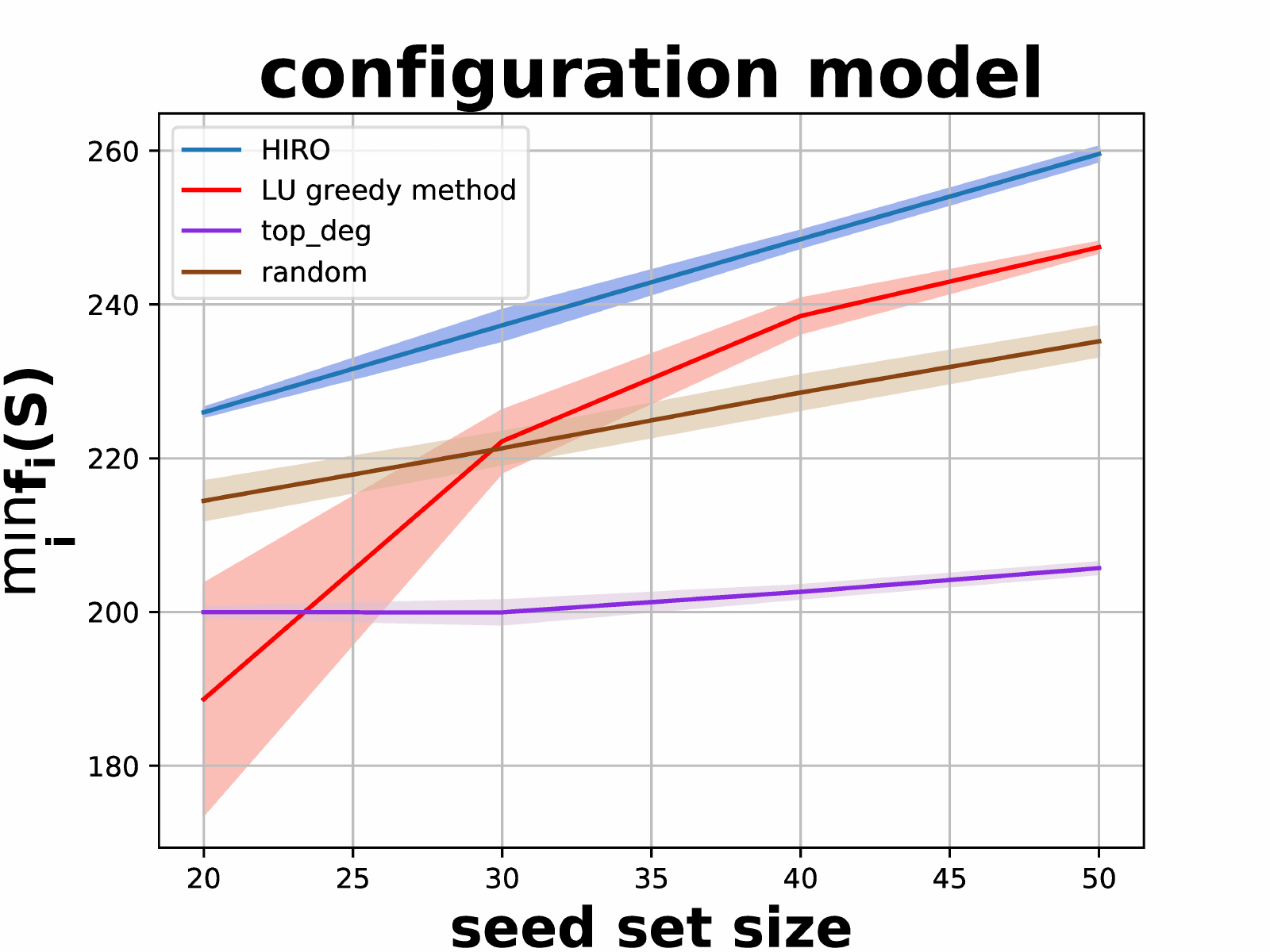} & 
  \includegraphics[width=1.4in, trim={0cm 0 1.64cm 0}, clip]{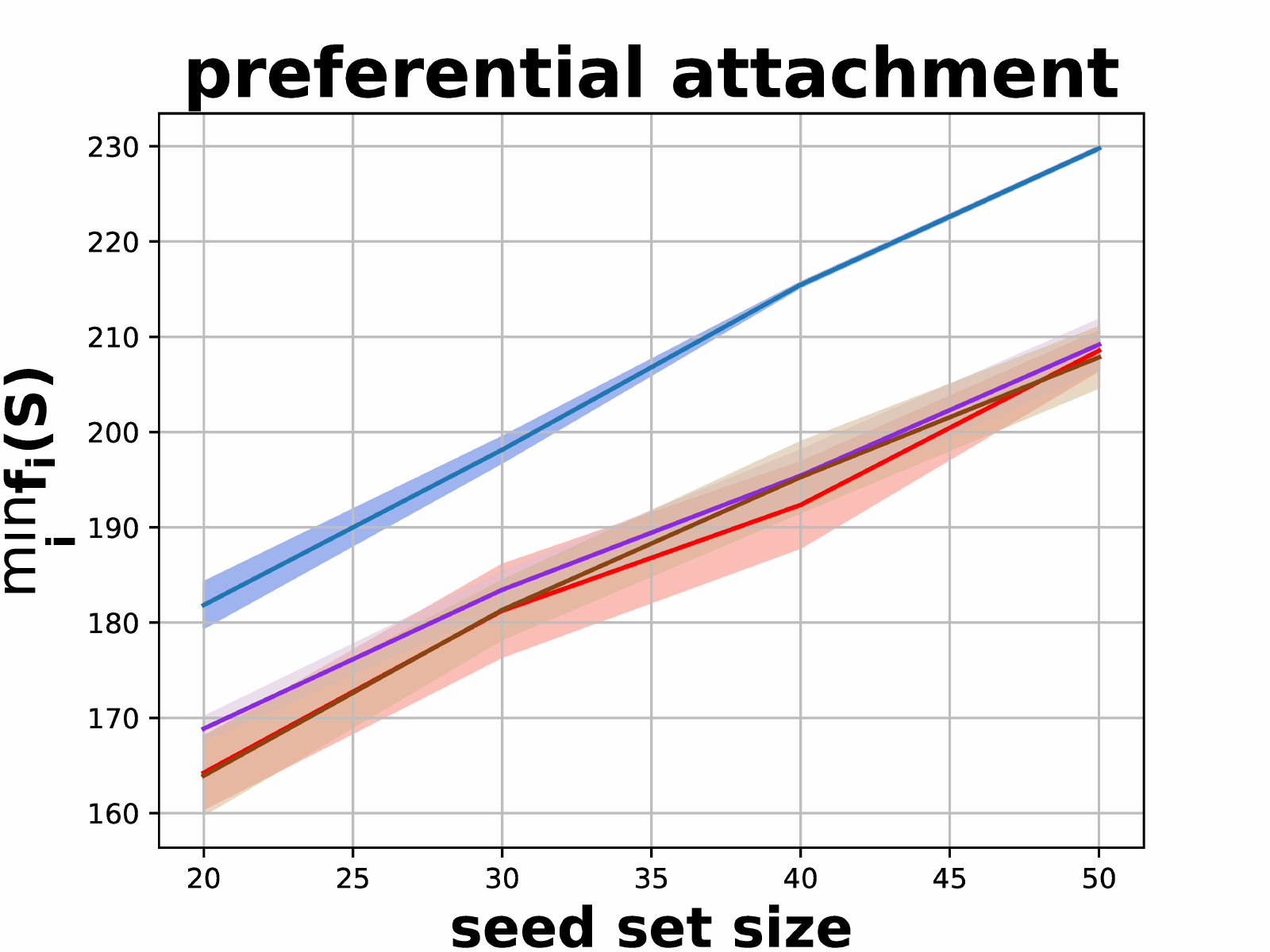} &
  \includegraphics[width=1.4in, trim={0cm 0 1.64cm 0}, clip]{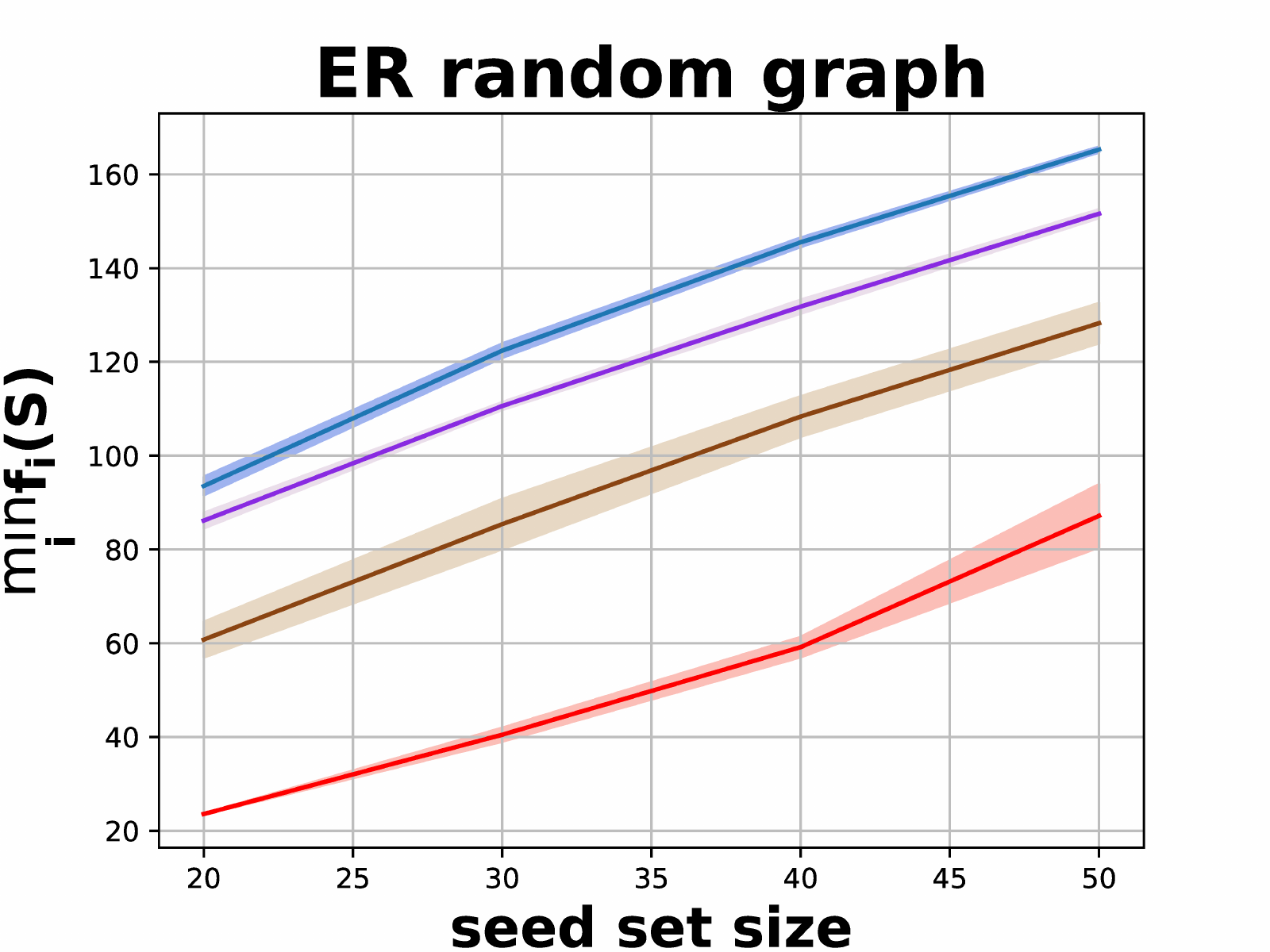} & 
  \includegraphics[width=1.4in, trim={0cm 0 1.64cm 0}, clip]{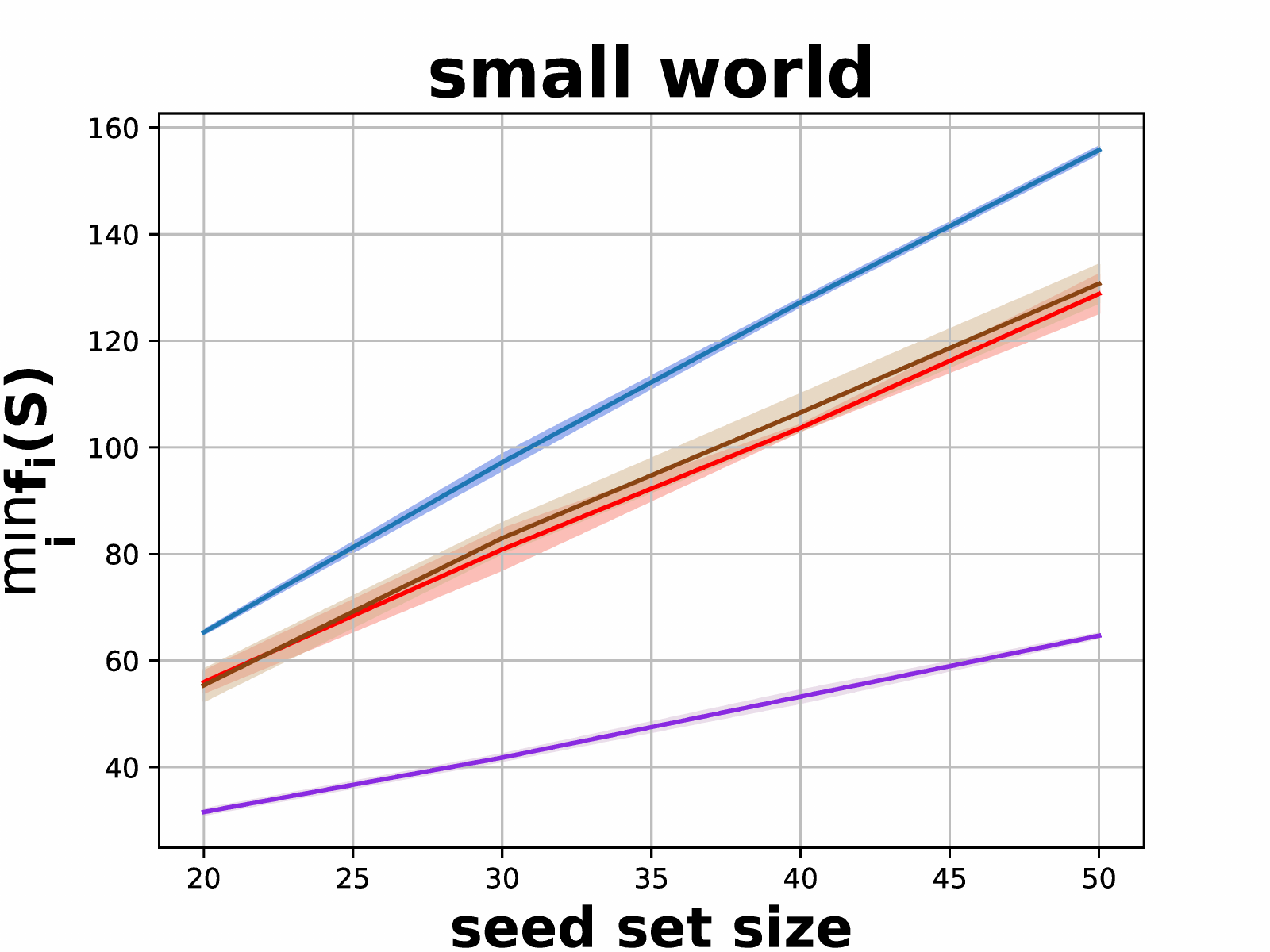} \\
  \end{tabular}
\end{center}
  \caption{Comparison to benchmark solutions of the robust problem.}
  \label{fig:ex3}
\end{figure*}

\begin{figure*}[h!t!]
\begin{center}
\begin{tabular} {cccc}
  \includegraphics[width=1.4in, trim={0 0 1.64cm 0}, clip]{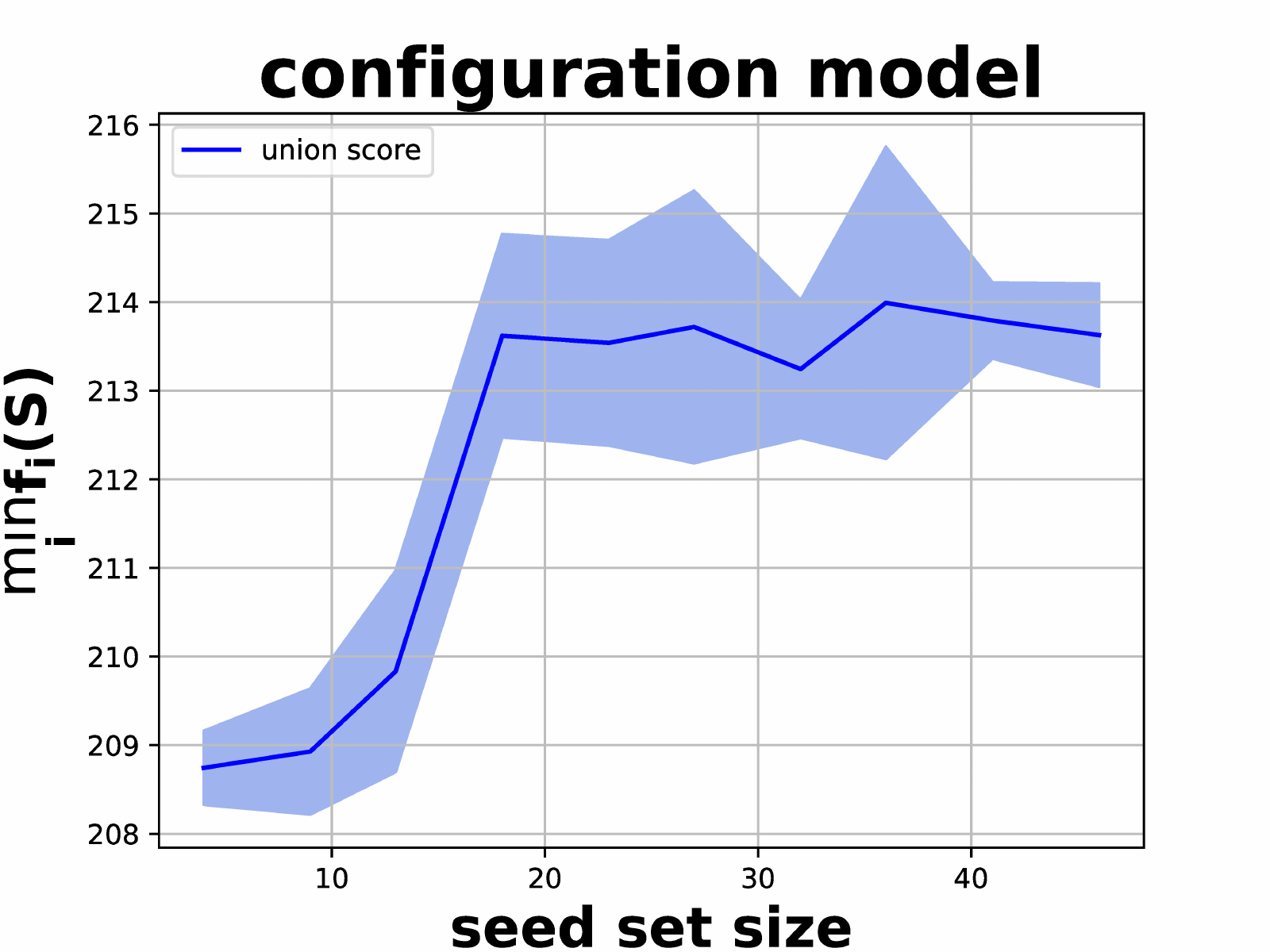} & 
  \includegraphics[width=1.4in, trim={0 0 1.64cm 0}, clip]{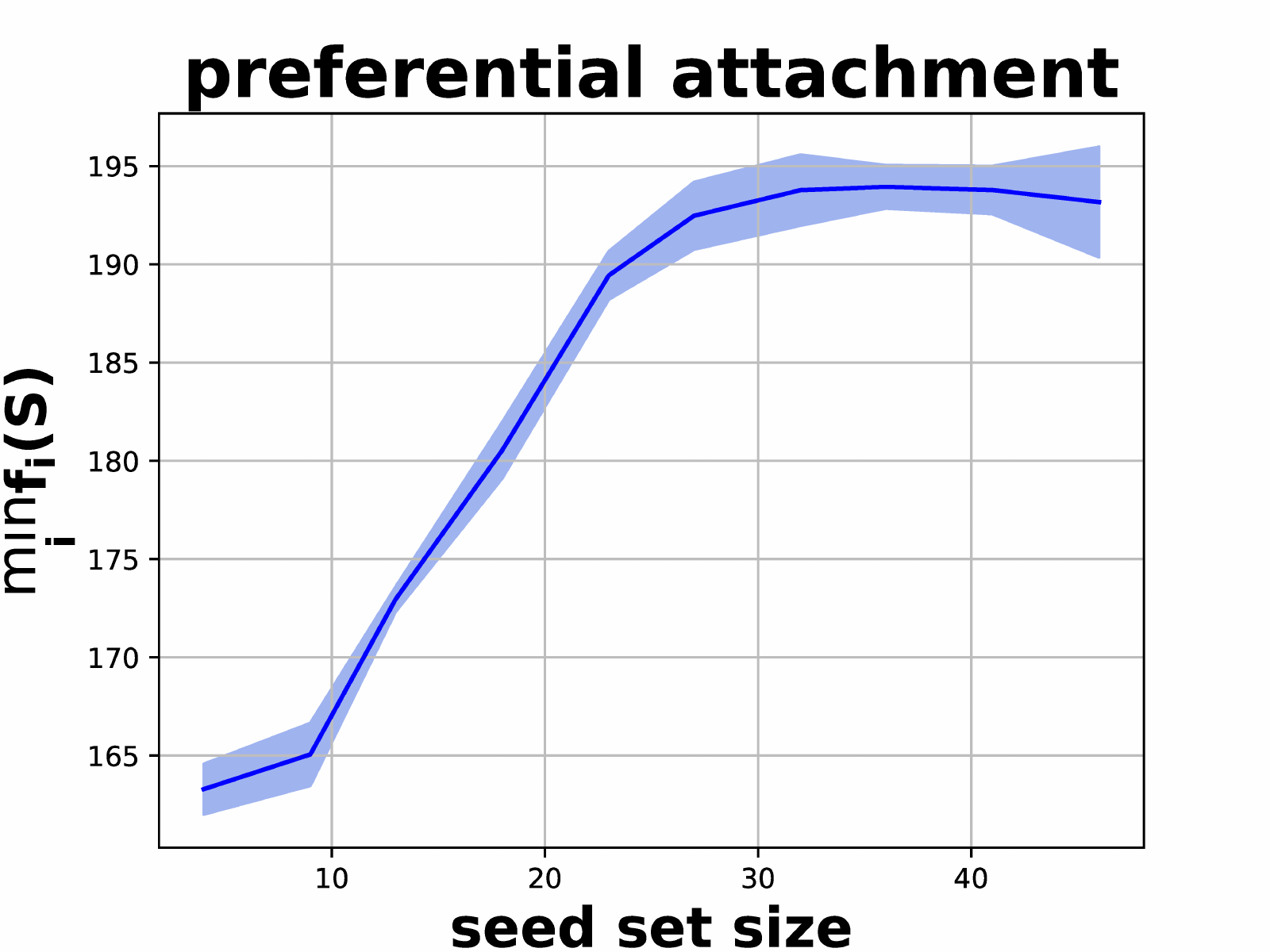} &
  \includegraphics[width=1.4in, trim={0 0 1.64cm 0}, clip]{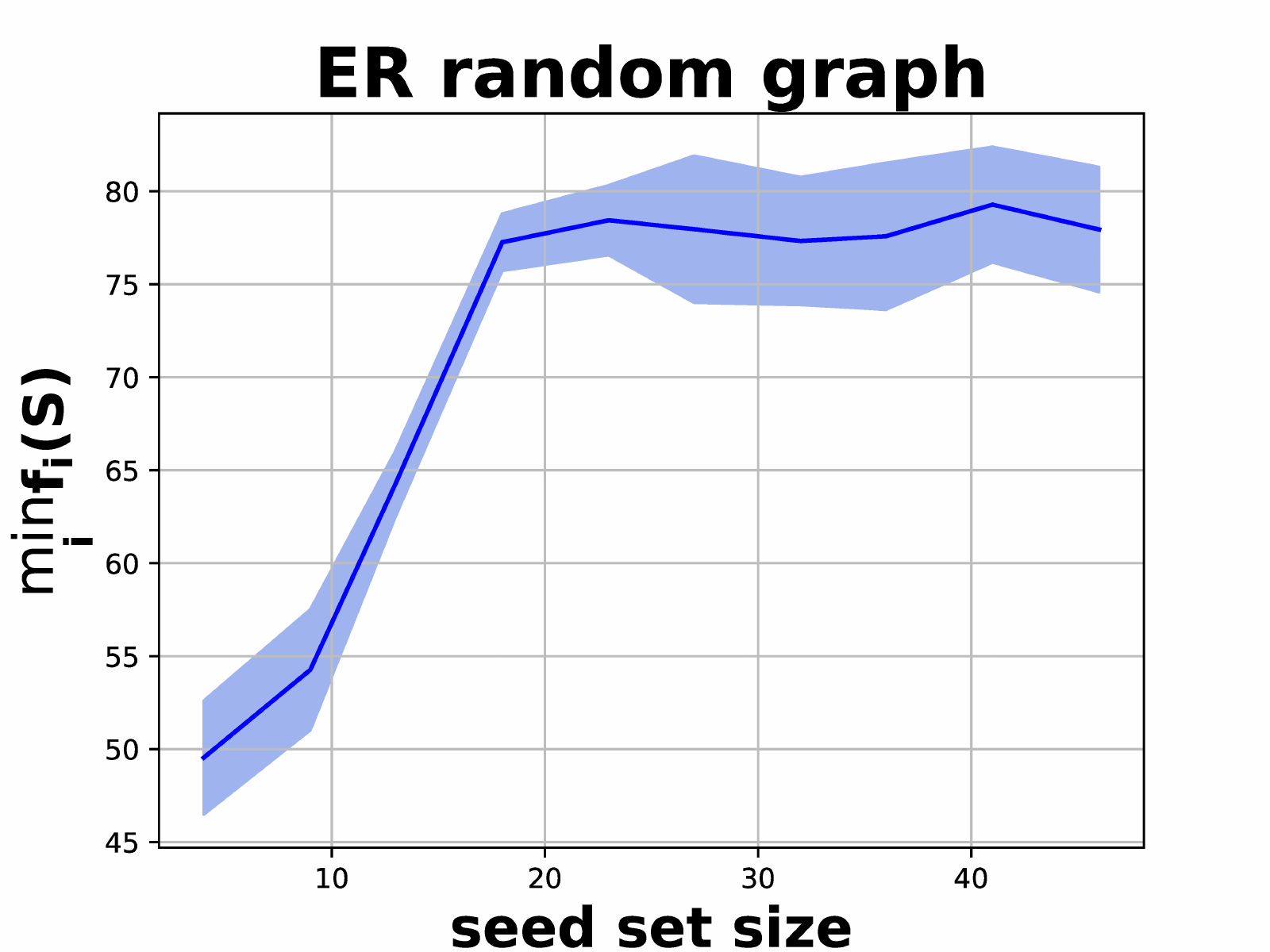} & 
  \includegraphics[width=1.4in, trim={0 0 1.64cm 0}, clip]{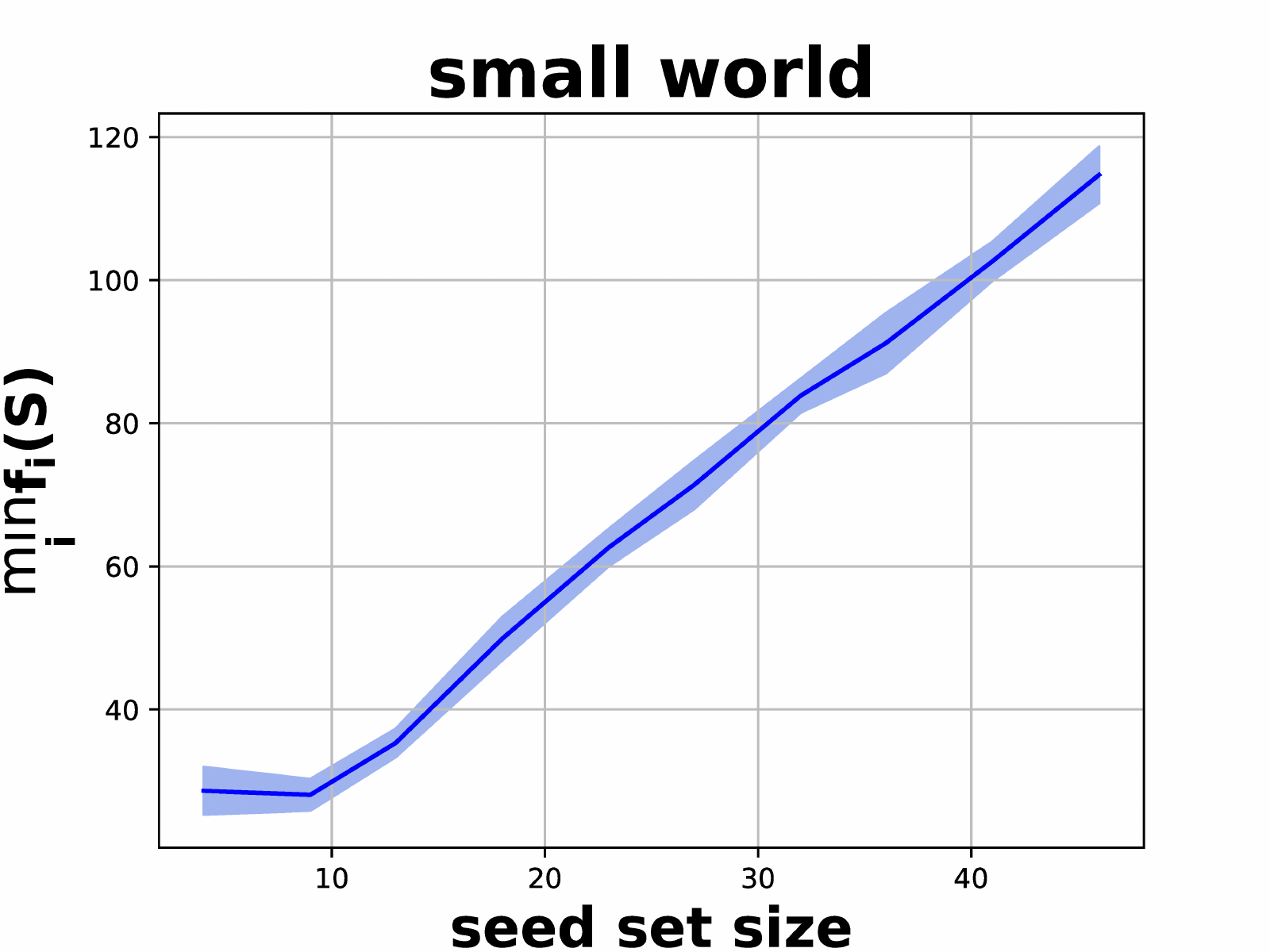} \\
  \end{tabular}
\end{center}
  \caption{Illustration of the gap between a regular seed set size and the bi-criteria augmented budget.}
  \label{fig:ex4}
\end{figure*}

%

%

To measure the empirical performance of the robust hyperparametric approach we conducted four sets of experiments: (1) First, we examine how many functions are empirically required for the sufficient covering of $\mathcal{F}=\{f_{\theta} \ | \ \theta \in \Theta \}$ and observe that in practice it is far smaller than the theoretical worst-case upper bound of Corollary \ref{cor:cover}. (2) We examine the rate of convergence of HIRO to the robust solution and show substantially faster convergence than in Theorem \ref{thm:robust}. (3) We benchmark the performance of HIRO against other methods for robust influence maximization and observed that it consistently outperforms all previous methods. (4) We measure how well a seed set of size $k$ performs compared to a seed set of slightly augmented budget according to the bi-criteria approximation guarantee (Corollary~\ref{cor:bicriteria}).


\paragraph{Graphs.} We generated four different synthetic networks using standard random graph models to analyze the impact of topological variations among different social networks. All networks were generated with $n=500$ vertices. We used the Barab\'{a}si-Albert preferential attachment model where each new node is connected to 4 (preferably high-degree) existing ones, the Watts-Strogatz small world model where each node is connected to 5 nodes in the ring topology and the probability of rewiring an edge is $3/n$, Erd\"{o}s-R\'enyi random graph model with edge-construction probability $p=3/n$ and the configuration model with a power law degree distribution and $\alpha=2$. For a more detailed description of these models please refer to Appendix C.


\smallskip
\noindent\textbf{Hyperparameric model.}
We used the sigmoid function as the hyperparameteric model to determine the diffusion probabilities, i.e. $h(\theta^\top x_e) = \frac{1}{1+\exp(-\theta^\top x_e)}$ as in~\cite{icml18}.  We generated $d$ random features in $[-1,1]$ for every edge. We used $d=5$, however our results are consistent across a large range of dimensions $d$ and feature-generating techniques, such as normal or uniform distributions over the unit hyper-cube  $[-1,1]^d$ and it's discrete analog $\{-1,1\}^d$. 
We sampled $\Theta_\epsilon = \{\theta_1,\ldots,\theta_l\}$ from $\Theta = [-1,1]^d$ and generated the family of influence functions $\F_\e = \{f_{i} \ | \ \theta_i \in \Theta_\epsilon \}$ for $l=20$. In addition we set $T=10$ HIRO iterations with the exception of Experiments 1 and 2, where $l$ and $T$ are the free variable, respectively.

\noindent\textbf{Benchmarks.}  We benchmark the performance of HIRO with respect to the following algorithms:
\begin{itemize*}
\item \emph{Random Seed:} We select $k$ nodes u.a.r.;  to account for variance, we average over the solutions from 100 trials.
\item \emph{Top k Degree nodes:} In this benchmark we chose the $k$ highest-degree nodes in the graph.
\item \emph{Random Greedy:} In robust optimization problems, a typical approach is to randomize over the best strategies against any possible influence function. Consequently, this method runs $\textsc{Greedy}$ algorithm on every function $f_i$ for $\theta_i \in \Theta_\varepsilon$ and  chooses one of the outputs uniformly at random.

\item \emph{ $\textsc{LUGreedy}$:} We compare ourselves against Algorithm 2 in ~\cite{wei} and refer to is as  $\textsc{LUGreedy}$. $\textsc{LUGreedy}$ is an algorithm specially developed for robust influence maximization.
 It is oblivious to the structure of the model and only accounts for the confidence interval of each edge. $\textsc{LUGreedy}$ produces two solutions,  by running the $\textsc{Greedy}$ algorithm twice - over both the lower and upper boundaries of the confidence intervals. Then, the algorithm chooses the best solution out of the two, assuming the lower boundaries as true probabilities. $\textsc{LUGreedy}$  achieves the best approximation for the robust ratio objective, mentioned in the introduction.
 
 
%

\end{itemize*}

\subsection{Experimental results}\label{sec:results}
We perform 50 trials for each experiment and plot mean and standard deviation in the figures. 

\noindent\textbf{Experiment 1.} (Figure \ref{fig:ex1}, Appendix~\ref{app:experiments})
The goal of this experiment is to understand the sample complexity of the hyperparameter, i.e. how many functions do we need to sample to approximate the robust solution accurately. We sample $l=50$ values from $\Theta$ as a benchmark (can be seen as a test/validation set). In that way we create a cover $\mathcal{F}_\varepsilon$ that serves as a proxy for $\F=\{f_\theta|\ \theta \in \Theta\}$. We plot $\min\limits_{1\le i\le l} f_i(S_r)$ where $S_r$ is the result of running HIRO over $r \in \{1, 10, 20, 30, 40, 50\}$ functions. We plot three such trials with different seed sizes $k \in \{10, 25, 50\}$. Figure \ref{fig:ex1}, deferred to the appendix, demonstrates that in practice sample complexity plateaus, even though theoretically increase in the number of functions should bring about an increase in the value of the robust solution. 

\noindent\textbf{Experiment 2.} (Figure \ref{fig:ex2},  Appendix~\ref{app:experiments})
We show that HIRO converges rapidly in practice, that is, even after a constant number of iterations, HIRO achieves great results. We run the experiment three times, for different values of $k\in\{10, 25, 50\}$. For each trial, we run HIRO for a $T\in \{1, 5, 10, 15\}$ number of iterations. Figure \ref{fig:ex2}, found in the appendix, illustrates this inquiry. Notwithstanding the small world model, which has slow growth, other plots indeed seem to converge quickly for all values of k. 

\noindent\textbf{Experiment 3.} (Figure \ref{fig:ex3})
This experiment compares HIRO with the benchmarks and illustrates consistency of our algorithm. We plot the value of a benchmark as a function of the size of the seed set. We can see that across all generative models and proposed benchmarks, even after few iterations, our algorithm performs at the very top, with the lowest variance. It is readily seen that other benchmarks are competitive when the seed set is small, but as the seed set grows so does the gap in performance of the best heuristic.  

\noindent\textbf{Experiment 4.} (Figure \ref{fig:ex4})
We evaluate the gap in the value of the robust solution between a seed set of size $k$ and a seed set of size $\beta\cdot k\log n$. 
Recall that the HIRO algorithm chooses at random among $T$ solutions a seed with size $k$. Instead, we take a union of these solutions and return a size-$k'$ subset of the union, where $\beta \in (0,1]$. Here we choose $k=10$ and we report the results in Figure~\ref{fig:ex4}.

%% file: conclusion.tex
\section{Conclusion}
\label{sec:conclusion}

In this paper, we proposed a new formulation of robust influence maximization by utilizing a very broad class of hyperparametric models. We provided an efficient reduction from continuous to discrete robust influence maximization and an optimal and computationally tractable algorithm for the problem in terms of bi-criteria approximation. We empirically assessed its performance and found that it consistently surpasses state-of-the-art methods.


%% file: appendix.tex
\appendix

%
%
%

\section{Omitted Proofs}
\label{app:proofs}

\subsection{Proof of Lemma~\ref{lemma:ratio_tight}}

\begin{proof}
Denote the solutions to the two different objectives as follows: $\hat{S}_r = \arg\max_{S:|S| \leq k}\min_{\p \in \mathcal{P}}\frac{f_{\p}(S)}{f_{\p}(S_{\p}^*)}$ and $\hat{S}_v = \arg\max_{S:|S| \leq k}\min_{\p \in \mathcal{P}}f_{\p}(S)$.

We will prove the lemma by contradiction. Specifically, let's assume that there exists a set of influence functions $\mathcal{P}$ for which $\min_{\p}f_{\p}(\hat{S}_r) < \frac{1}{\sqrt{n}}\min_{\p} f_{\p}(\hat{S}_v)$, i.e. for this $\mathcal{P}$, the solution for the robust ratio objective is suboptimal with respect to the total number of nodes influenced by a factor greater than $\sqrt{n}$.

To ease the notation let us denote with $f_{r}$ the function that achieves

Hence it holds:
\begin{multline}\label{eq:eq1}\sqrt{n} \cdot f_r(\hat{S}_r) = \sqrt{n} \cdot \min_{\p}f_{\p}(\hat{S}_r)\\ < \min_{\p} f_{\p}(\hat{S}_v) = f_v(\hat{S}_v) \leq f_r(\hat{S}_v)\end{multline}

where the last inequality is due to the minimality of $f_v$. Let us denote with $f_m$ the function that has the minimum ratio for $\hat{S}_v$. That is, $\frac{f_m(\hat{S}_v)}{f_m(S^*_m)} = \min_{\p}\frac{f_{\p}(\hat{S}_v)}{f_{\p}(S_{\p}^*)}$. Then,

\begin{multline}\label{eq:eq2}\frac{1}{\sqrt{n}} > \frac{f_r(\hat{S}_r)}{f_r(\hat{S}_v)} \geq \frac{f_r(\hat{S}_r)}{f_r(S_r^*)}\\
= \max_S\min_{\p}\frac{f_{\p}(S)}{f_{\p}(S_{\p}^*)} \geq \min_{\p}\frac{f_{\p}(\hat{S}_v)}{f_{\p}(S_{\p}^*)} = \frac{f_m(\hat{S}_v)}{f_m(S^*_m)}\end{multline}

where the last inequality holds due to the maximality of $\hat{S}_r$. Now we can prove a contradiction as follows:

$$f_m(S^*_m) > \sqrt{n} \cdot f_m(\hat{S}_v) \geq \sqrt{n} \cdot f_v(\hat{S}_v) $$
$$> n\cdot f_r(\hat{S}_r) \geq n$$

The first inequality is due to (\ref{eq:eq2}), the second is due to the fact that $f_v(\hat{S}_v) = \arg\min_{\p}f_{\p}(\hat{S}_v)$, while the third is from (\ref{eq:eq1}). Finally, since $|\hat{S}_r| \geq 1$ the influence function is also at least 1 (at least all the nodes in $\hat{S}_r$ get influenced).

Now notice that the influence of any set of nodes cannot be more than $n$ and as a result we have a contradiction. Thus, $\min_{\p}f_{\p}(\hat{S}_r) \geq \frac{1}{\sqrt{n}}\min_{\p} f_{\p}(\hat{S}_v)$.

The graph in Figure~\ref{fig:ratio} shows that there exist a set $\mathcal{P}$ for which $\min_{\p}f_{\p}(\hat{S}_r) = \Omega\left(\frac{1}{\sqrt{n}}\right)\min_{\p} f_{\p}(\hat{S}_v)$ which concludes the proof.

\end{proof}

\subsection{Proof of Lemma~\ref{lemma:tight_example}}

Consider a cycle on $n$ nodes, connected with edges of diffusion probability $1-\lambda$, and an additional center node $v^\star$ that is connected to all the nodes of the cycle. Notice that the number of edges is $m = 2n$.  To consider the Lipschitzness of the influence function on this graph we consider the change in the influence of $v^\star$ in case that the probabilities connecting it to the cycle are all $\lambda$ and the case in which they are all $\e = n\cdot\lambda$.

The influence of $v^\star$ in the first case is at most $n(1 - (1-\lambda)^n)$. For sufficiently large $n$:
$$(1-\lambda)^n  = \left(1-\frac{n\lambda}{n}\right)^n \approx e^{-n\lambda} \approx 1 - n\lambda$$

So, the influence is at most $n^2\lambda = n\e$. Once the probabilities on edges connecting it to the cycle increase from $\lambda$ to $\e$ its expected influence becomes at least $n(1 - (1-\e)^n)(1-\lambda)^n$. Using the same approximation as before for sufficient large $n$, we get that the influence of $v^\star$ is at least $n^2\e(1 - n\lambda) = n^2\e(1 - \e) = n^2\e - n^2\e^2$.

Then, $|f_\textbf{p}(v^\star) - f_{\textbf{p'}}(v^\star)| = n^2\e  - n^2\e^2 - n\e$. After setting $\e = \frac{1}{n}$, this bound becomes $n^2\e - 2n\e = (n\frac{m}{2} - 2n)\e$. Thus, for small $\e$ the Lipschitz constant is asymptotically achieved with $n$.

Notice that this example can be simplified if we use probabilities of 1 in the cycle, and 0 and $\e$ in the the connections of $v^\star$. The reason why we avoided the values $0, 1$ is because for some generalized linear models, e.g. in the logistic or the probit model the values of the probabilities are strictly in $(0,1)$ instead of $[0,1]$.

\subsection{Proof of Lemma~\ref{lemma:reduction}}


In Definition~\ref{def:cover} we defined an $\e$-cover of $\F$ as a set $\F_\e \subset F$ s.t. for any $f_\q \in \mathcal{F}$ there exists a function $f_j \in \F_\e$ such that: $|f_\q(S) - f_j(S)| \leq \e$ for all $S\subseteq V$. Using this definition we can proceed as follows:

\vspace{-.2cm}
$$\forall S\subseteq V,\, \forall f_\q \in \mathcal{F}, \,\exists f_j \in \F_\e: |f_{\q}(S) - f_{j}(S)| \leq \e$$
$$\Rightarrow -\min_{f_\q \in\mathcal{F}} f_\q(S) + f_{j}(S) \leq \e$$
$$\Rightarrow \min_{f_\q \in \mathcal{F}} f_\q(S) \geq \min_{f_i \in F}f_{i}(S) - \e$$


Simultaneously it holds that:
\vspace{-.2cm}
$$\forall S\subseteq V,\, \min_{f_i \in \F_\e}f_{i}(S) - \min_{f_\q \in \mathcal{F}}f_{\q}(S) \geq 0$$

\vspace{-.1cm}
since $\F_\e \subset \mathcal{F}$. Let $S^* = \arg\max_{S: |S| \leq k}\min_{f_\q \in \mathcal{F}}f_{\q}(x)$ and $S^*_\e = \arg\max_{S: |S| \leq k}\min_{f_i \in \F_\e}f_{i}(x)$. Then it is:

$$\min_{f_i \in \F_\e}f_{i}(S^*_\e) - \min_{f_\q \in \mathcal{F}}f_{\q}(S^*) \geq$$
$$\min_{f_i \in \F_\e}f_{i}(S^*) - \min_{f_\q \in \mathcal{F}}f_{\q}(S^*) \geq 0$$
$$\Rightarrow \max_{S: |S| \leq k}\min_{f_i \in \F_\e}f_{i}(S) \geq \max_{S: |S| \leq k}\min_{f_\q \in \mathcal{F}}f_{\q}(S)$$

Hence, utilizing an algorithm that guarantees an $\hat{S} \subseteq V$ such that:
$$\min_{f_i\in \F_\e}f_i(\hat{S}) \geq \alpha\cdot \max_{S: |S|\leq k}\min_{f_i\in \F_\e}f_i(S)$$
we get that for the family $\mathcal{F}$ it holds:
$$\min_{f_\q \in \mathcal{F}}f(\hat{S}) \geq \alpha\cdot \max_{S: |S| \leq k}\min_{f_\q \in \mathcal{F}}f(S) - \e.$$\qed

\section{Lower Bound}
\label{app:lower_bound}

We build a similar reduction to the one in \cite{kempe}, reducing from is GAP SET COVER.


In a SET COVER instance we have a universe elements $U = \{u_1, u_2, \ldots, u_\ell\}$ and a collection of subsets of $U$, $\mathcal{T} = \{T_1, T_2, \ldots, T_M\}$, where $T_i \subseteq U$ for all $i \in [M]$. The goal is to find a cover $C \subseteq \mathcal{T}$ such that $\cup_{T \in C}T = U$ and the size of $C$ is minimized. In the decision version of the problem we are also given an integer $k$ and we are asked whether the optimal solution has value $|C| \leq k$ or $|C| > k$. The GAP SET COVER is a slightly stronger problem that asks whether there is a solution $C$ such that  $|C| \leq k$ or $|C| > (1 -\delta)\log Nk$, for any $\delta \in (0,1)$. We will assume that $k \leq \min\{M, \ell\}$, otherwise we can always find the optimal solution by simply picking all the elements of $\mathcal{T}$ or at least one set per element of $U$ that contains it (assuming that a set cover exists, such a set always exists as well). Both problems are $NP$-hard as proved in \cite{set_cover} and \cite{dinur}.

For any given instance of GAP SET COVER we construct an instance of robust influence maximization (RIM) by constructing a graph on $n$ nodes, and $\ell$ different influence functions. The goal is to maximize the influence with respect to the worse influence function. Each influence function is associated with a different set of diffusion probabilities. We will prove that if we can find a seed set that is a better than $\frac{1}{n^{1-\e}}$-approximation to the maximin solution of this RIM problem, then we can solve gap set cover.

\begin{figure*}[h!t!]
\begin{center}
\begin{tabular} {cccc}
  \includegraphics[width=1.45in, trim={0 0 1.64cm 0}, clip]{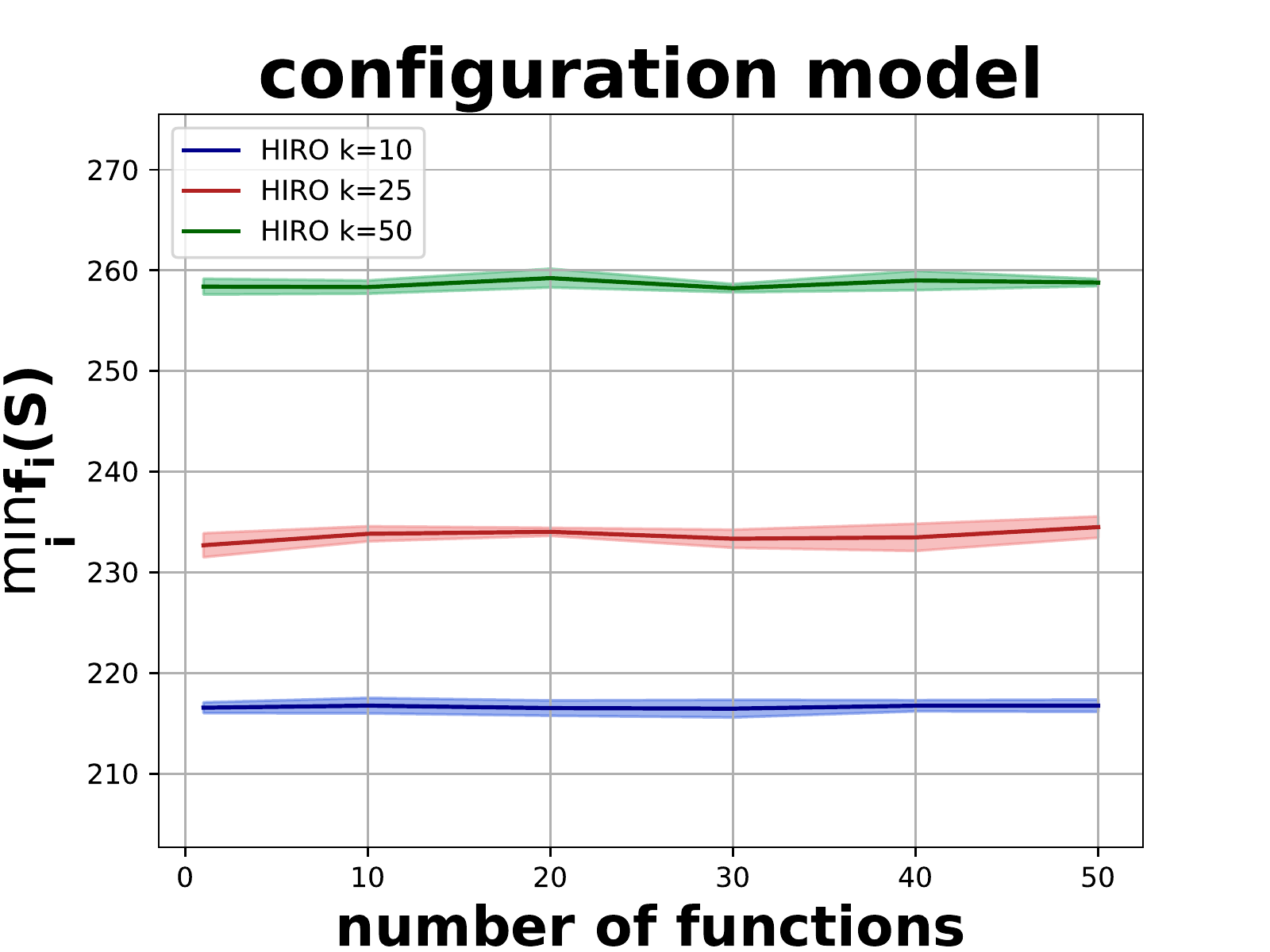} & 
  \includegraphics[width=1.45in, trim={0 0 1.64cm 0}, clip]{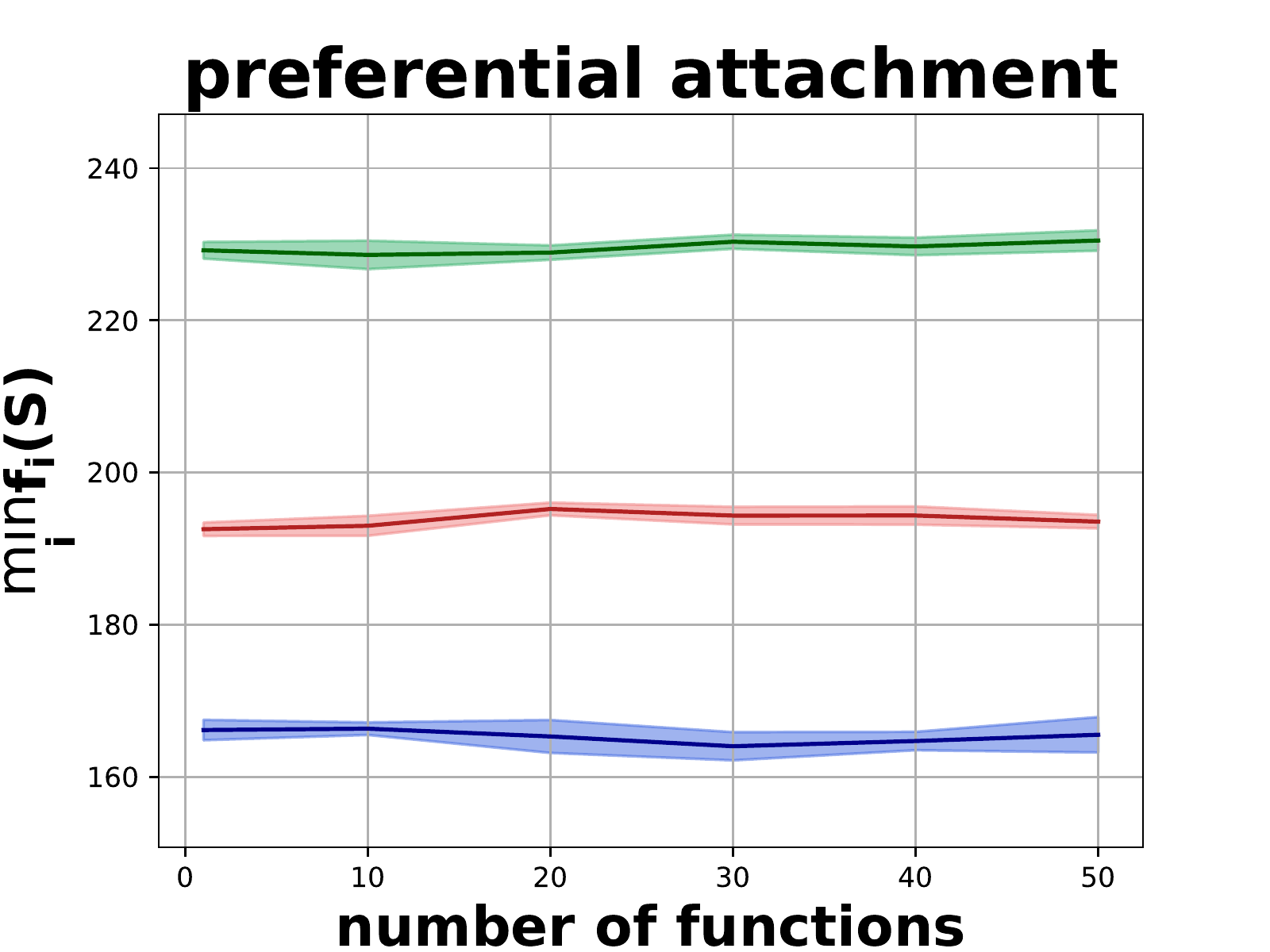} &
  \includegraphics[width=1.45in, trim={0 0 1.64cm 0}, clip]{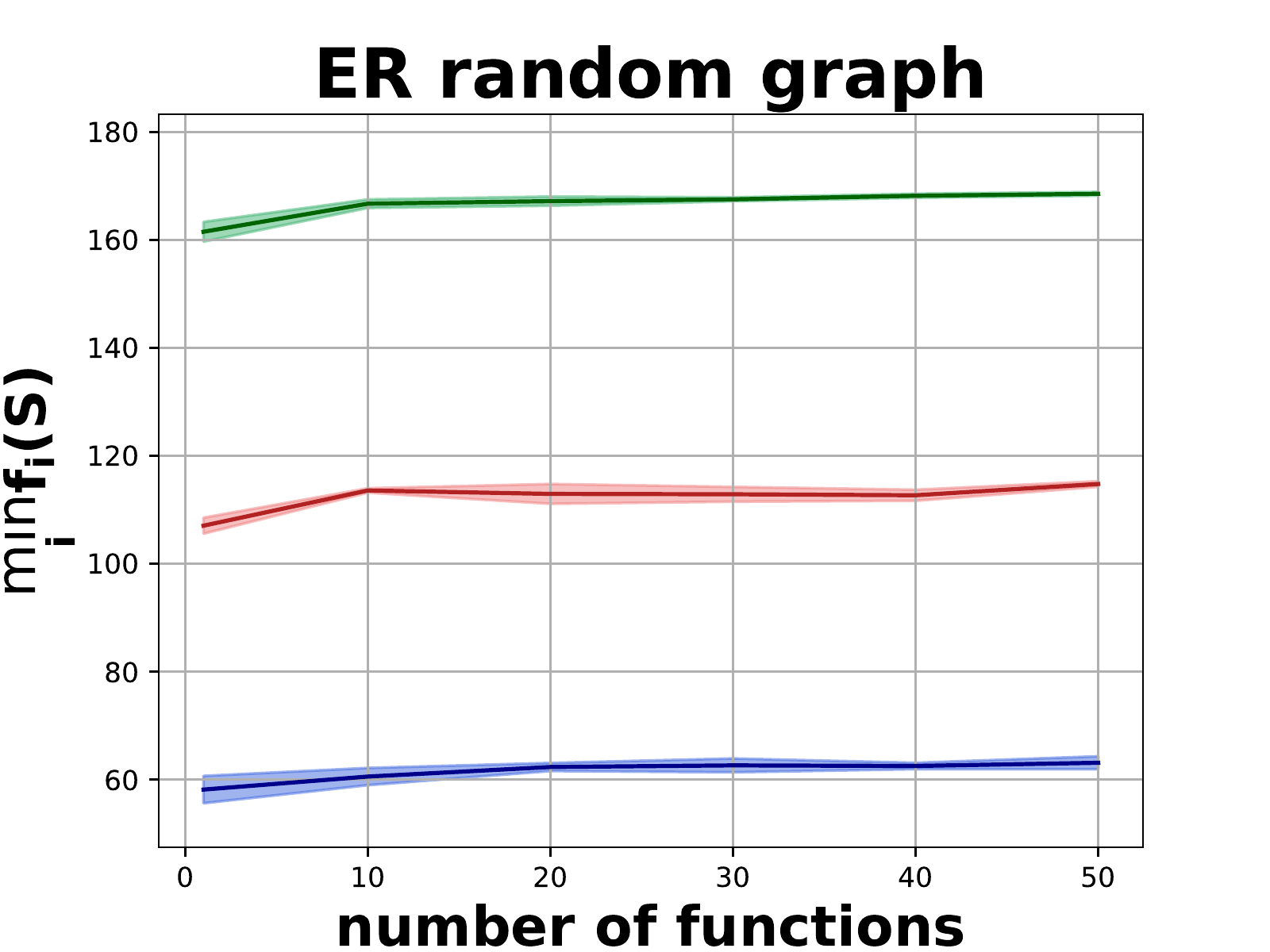} & 
  \includegraphics[width=1.45in, trim={0 0 1.64cm 0}, clip]{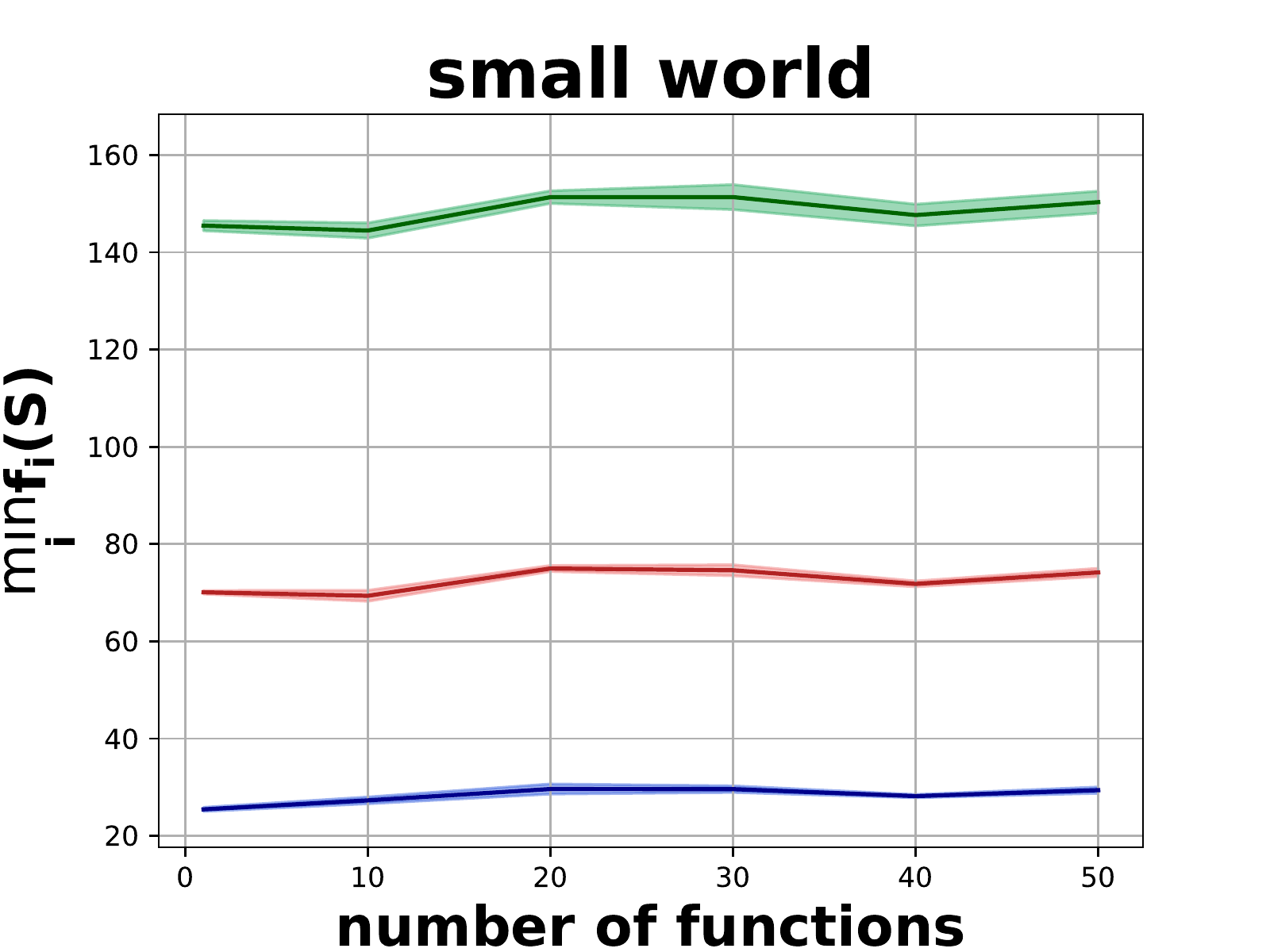} \\
  \end{tabular}
\end{center}
  \caption{Number of functions needed to cover the Hyperparameter's space.}
  \label{fig:ex1}
\end{figure*}

\begin{figure*}[h!t!]
\begin{center}
\begin{tabular} {cccc}
  \includegraphics[width=1.45in, trim={0 0 1.64cm 0}, clip]{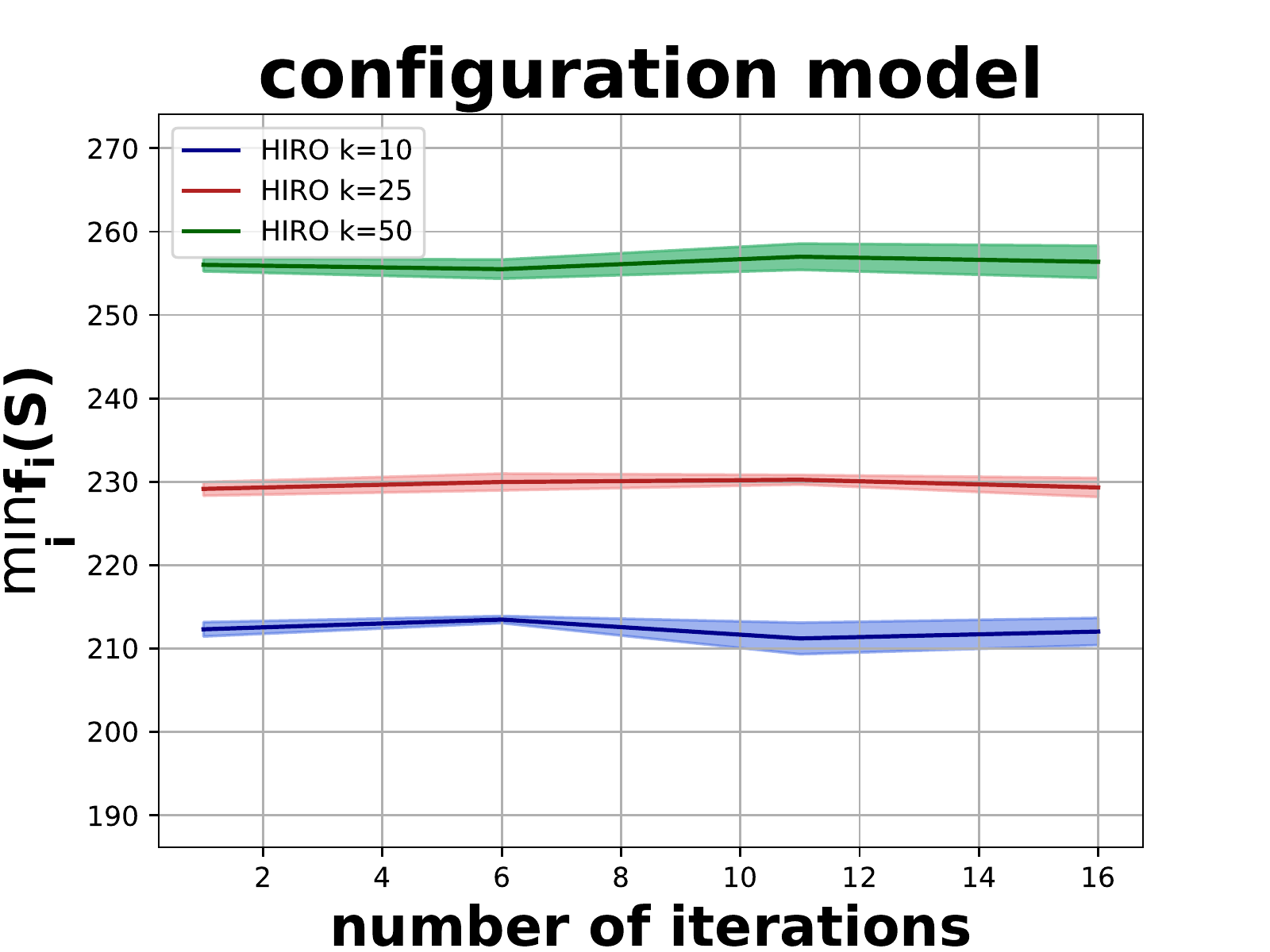} & 
  \includegraphics[width=1.45in, trim={0 0 1.64cm 0}, clip]{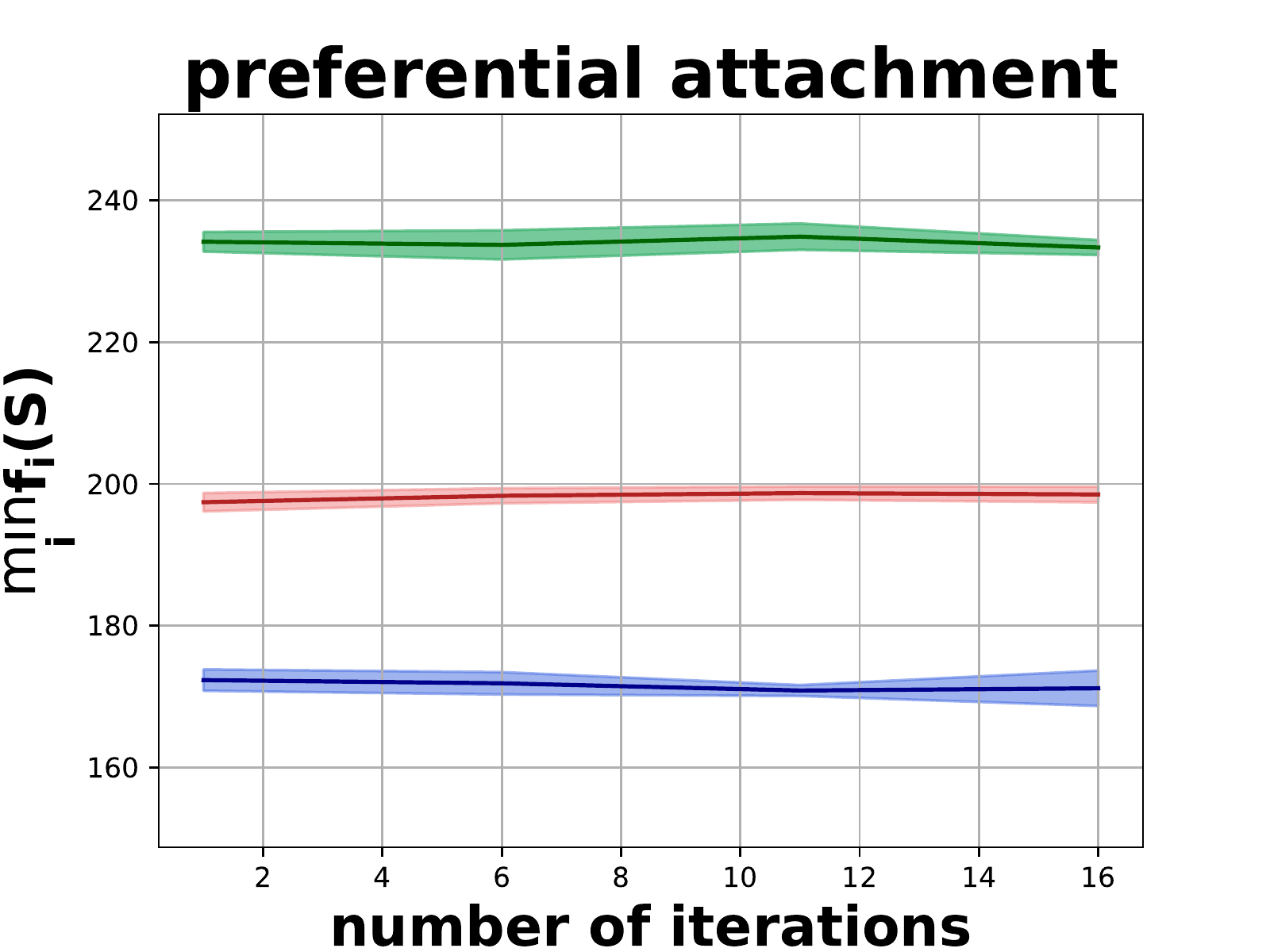} &
  \includegraphics[width=1.45in, trim={0 0 1.64cm 0}, clip]{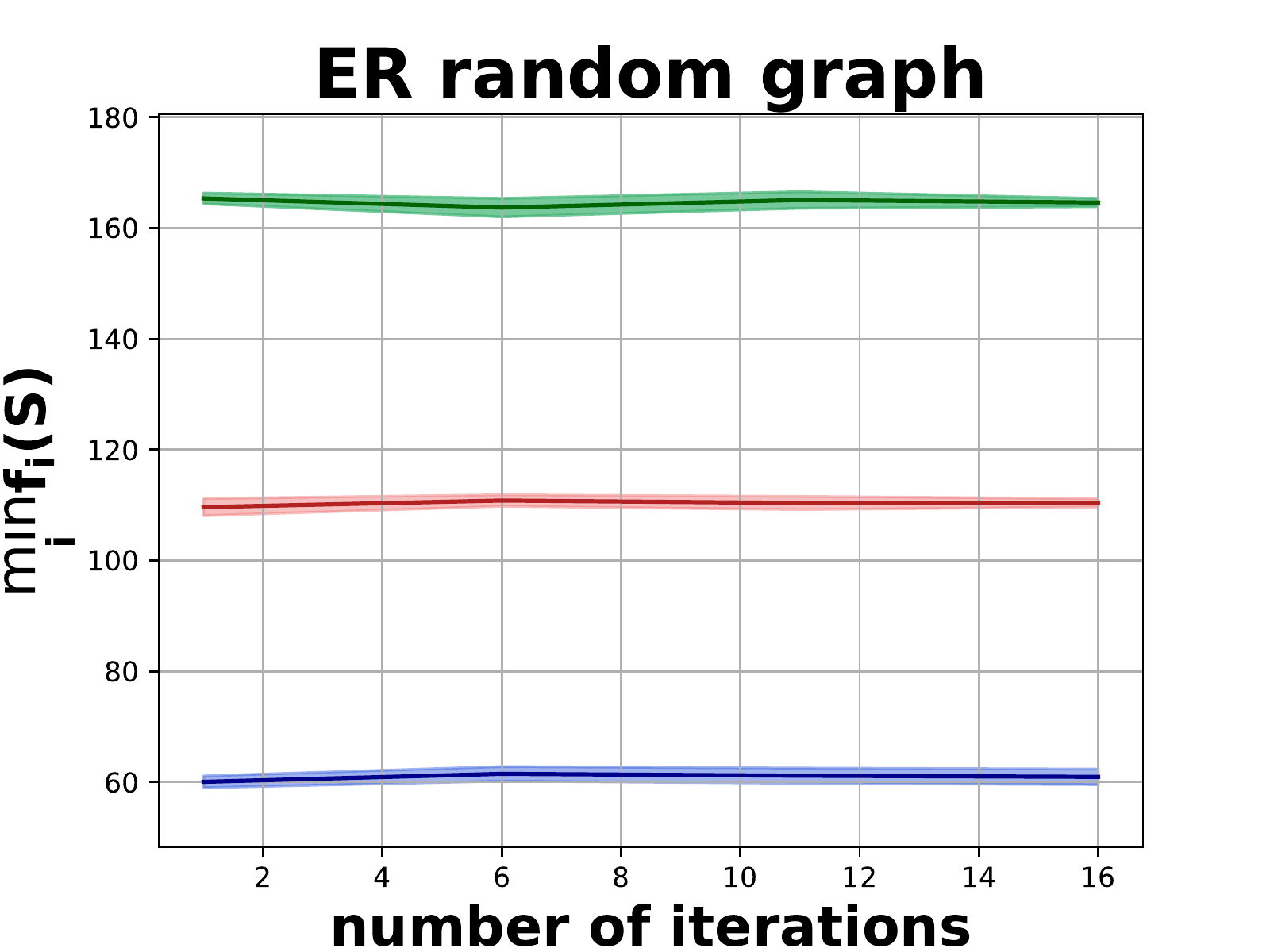} & 
  \includegraphics[width=1.45in, trim={0 0 1.64cm 0}, clip]{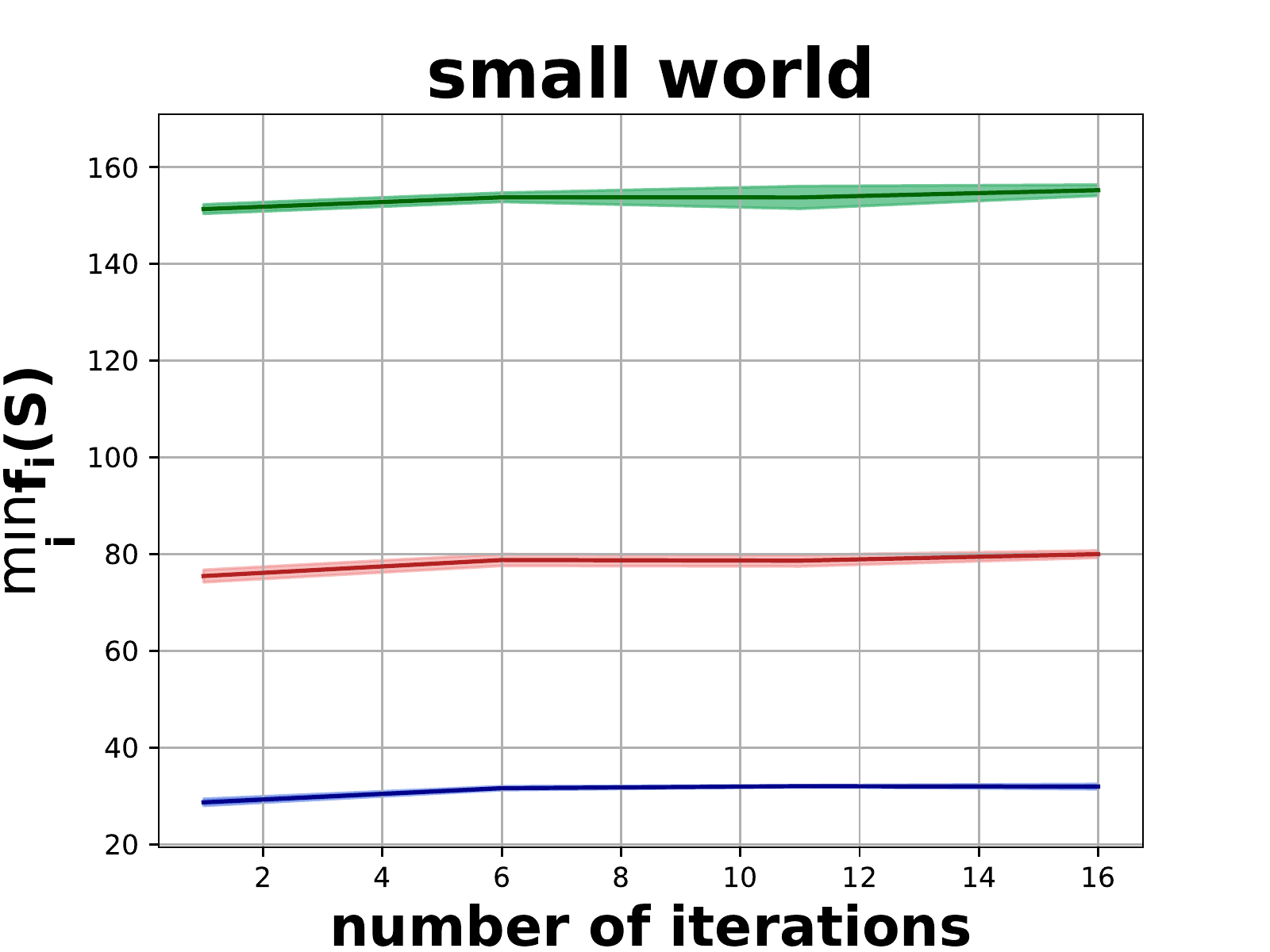} \\
  \end{tabular}
\end{center}
  \caption{Speed of convergence of HIRO for multiple seed set sizes.}
  \label{fig:ex2}
\end{figure*}

We construct the following bipartite graph with vertex set $V = A\cup B$. The set $A$ contains exactly $M$ nodes, one node $a_T$ for each $T$, $i \in [M]$. The set $B$ contains $m$ nodes ($m$ to be fixed later in the proof) for each element $u \in U$: $\{b_{u1}, b_{u2}, \ldots, b_{um}\}$, so $m\ell$ nodes in total. The total number of nodes in the graph is $n = M + m\ell$. 

We create the edges of the graph according to the the set cover solution $C$. For every $T \in C$ we add the directed edges from $a_T$ to $\{b_{u1}, b_{u2}, \ldots, b_{um}\}$ for all $u \in T$. That is $m|T|$ edges per element $T\in C$.

Each influence function induces different probabilities on the edges. We have $\ell$ functions. For the $u^{th}$ function, set the probability of the edges $\{(a_T, b_{u1}), (a_T, b_{u2}), \ldots, (a_T, b_{um})\}$ for which $u \in T$ to $1-\lambda$ and the probability of the rest of the edges to $\lambda$.

There are two cases: $|C| \leq k$ and $|C| > (1-\delta)\ln\ell k$. Let us focus on the case where $|C| \leq k$ first. One can easily see that if we choose the $a_T$s for which $T \in C$ as seeds, we can achieve expected diffusion of at least $|C| + (1-\lambda)m$ on each of the $\ell$ influence functions due to the fact that every $j \in [\ell]$, $u_j$ is covered by the solution $C$, i.e. there exists $T$ such that $u_j \in T$. Thus, in this case $\max_S\min_{j \in [\ell]}f_j(S) \geq |C| + (1-\lambda)m$.

In the second case, there is no cover of size at most $(1-\delta)\ln\ell k$. However, we are allowed to choose at most $(1-\delta)\ln\ell k$ as seeds. Hence, for any choice of seeds there is definitely an element $u_j \in U$ that is not covered. As a result, for the $j^{th}$ influence function the expected number of influenced nodes is at most $(1-\delta)\ln\ell k \cdot(1 + m(\ell-1)\lambda)$. That is because each node is connected to at most $m(\ell-1)\lambda$ other nodes (since it is definitely not connected to $u_j$) and there are no high probability edges that are triggered in this function. As a result $\max_S\min_{j \in [\ell]}f_j(S) \leq (1-\delta)\ln\ell k \cdot(1 + m(\ell-1)\lambda)$.

We want to be able to distinguish between the first and the second case, i.e. the case where $\max_S\min_{j \in [\ell]}f_j(S) \geq |C| + (1-\lambda)m$ and when $\max_S\min_{j \in [\ell]}f_j(S) \leq (1-\delta)\ln\ell k \cdot(1 + m(\ell-1)\lambda)$. To this end, we consider the ratio: $\frac{(1-\delta)\ln\ell k(1+m(\ell-1)\lambda)}{|C| + (1-\lambda)m}$ which we want to prove that  is less than $\frac{1}{n^{1-\e}}$ for any $\e >0$. Remember that for the number of nodes in the graph it holds $n = M + m\ell$.

Hence, if $\frac{ |C| + (1-\lambda)m}{n^{1-\e}}>(1-\delta)\ln\ell k(1+m(\ell-1)\lambda)$ we will be able to separate the two cases.

First assume $\min\{M,\ell\}=\ell$:
$$\frac{ |C| + (1-\lambda)m}{(M + m\ell)^{1-\e}} \geq (1-\lambda)\frac{1+m}{((1+m)M)^{1-\e}}$$
$$ \geq (1-\lambda)\frac{m^\e}{M^{1-\e}}$$

Now for $m=M^{3/\epsilon}, \lambda=1/m$ we have that as $M$ grows
$$\left(1-\frac{1}{m}\right)\frac{m^{\epsilon}}{M^{1-\epsilon}} > (1-\delta)M^{2+\epsilon}>$$
$$ (1-\delta)M^2\ln M$$
 which is asymptotically larger than $(1-\delta)\ln\ell k(1+m(\ell-1)\lambda)$. Choosing $m = \ell^{3/\e}$ solves the other case. Hence setting $m = (\max\{M, \ell\})^{3/\e}$ and $\lambda=1/m$ completes the reduction. Hence, if we have a better than $\frac{1}{n^{1-\e}}$-approximation algorithm for robust influence maximization then we can solve gap set cover.

\section{Omitted Details from Experiments}
\label{app:experiments}

\textbf{Synthetic Graphs:} As we discussed in Section~\ref{sec:experiments} different graph models yield graphs with different topological properties. The ones we selected for our experiments are the following: 
{\leftmargini=2ex
\begin{itemize*}

\item \emph{Small-World network:} In this model most nodes are not neighbors of one another, but the path from each node to another is short. Specifically we use the Watts-€"Strogatz model that is known for its high clustering coefficient and small diameter properties. Each node is connected to 5 nodes in the ring topology and the probability of rewiring an edge is $1/n$ where $n$ is the number of nodes in the graph. We work with graphs of sizes 100-250.

\item \emph{Preferential Attachment (Barab\'{a}si-Albert):} The degree distribution of this model is a power law and hence captures interesting properties of the real-world social networks. We took 2 initial vertices and added 2 edges at each step, using the preferential attachment model, until we reached 100-250 vertices.

\item \emph{Configuration model:} The configuration model allows
us to construct a graph with a given degree distribution. We chose 100-250 vertices and a power-law degree distribution with parameter $\alpha = 2$.

\item \emph{Erd\"{o}s-R\'enyi:} We used the celebrated $G(n,m)$ model to create a graph with 100-250 vertices and edges with probability $p=3/n$. $G(n,m)$ does not capture some of the properties of real social networks, however it is a very impactful model with variety of applications in several areas of science.
\end{itemize*}